\documentclass[11pt]{article}
\usepackage{amsmath, amsfonts, amssymb, amsthm,  graphicx, enumerate}
\usepackage[letterpaper, left=1truein, right=1truein, top = 1truein, bottom = 1truein]{geometry}

\usepackage[authoryear,round]{natbib}
\usepackage[hidelinks]{hyperref}
\hypersetup{
	colorlinks=true,
	linkcolor={blue},
	citecolor=blue
}
\usepackage{showlabels}
\usepackage[usenames]{color}
\usepackage{multirow}

\usepackage{algorithm,algorithmic}

\usepackage{prettyref, xspace}

\newcommand{\bmu}{\boldsymbol{\mu}}

\newcommand{\SR}{\mathrm{SR}}
\newcommand{\BSR}{\mathrm{BSR}}
\newcommand{\PB}{\mathrm{PB}}
\newcommand{\BPB}{\mathrm{BPB}}

\newcommand{\wh}{\widehat}
\newcommand{\wt}{\widetilde}
\newcommand{\bem}{\begin{bmatrix}}
\newcommand{\eem}{\end{bmatrix}}

\newcommand{\Expect}{\mathbb{E}}

\newcommand{\Prob}{\mathbb{P}}

\newcommand{\argmax}{\mathop{\rm argmax}}

\theoremstyle{remark}

\theoremstyle{plain}

\newtheorem{lemma}{Lemma}
\newtheorem{theorem}{Theorem}
\newtheorem{assumption}[]{Assumption}

\theoremstyle{definition}

\theoremstyle{plain}

\newtheorem{coro}{Corollary}

\newrefformat{eq}{(\ref{#1})}
\newrefformat{chap}{Chapter~\ref{#1}}
\newrefformat{sec}{Section~\ref{#1}}
\newrefformat{algo}{Algorithm~\ref{#1}}
\newrefformat{fig}{Fig.~\ref{#1}}
\newrefformat{tab}{Table~\ref{#1}}
\newrefformat{rmk}{Remark~\ref{#1}}
\newrefformat{clm}{Claim~\ref{#1}}
\newrefformat{def}{Definition~\ref{#1}}
\newrefformat{cor}{Corollary~\ref{#1}}
\newrefformat{lmm}{Lemma~\ref{#1}}
\newrefformat{lemma}{Lemma~\ref{#1}}
\newrefformat{prop}{Proposition~\ref{#1}}
\newrefformat{app}{Appendix~\ref{#1}}
\newrefformat{ex}{Example~\ref{#1}}
\newrefformat{exer}{Exercise~\ref{#1}}
\newrefformat{soln}{Solution~\ref{#1}}
\newrefformat{cond}{Condition~\ref{#1}}

\newcommand{\lunder}[1]{{\underset{\raise0.3em\hbox{$\smash{\scriptscriptstyle-}$}}{#1}}}

\newcommand{\indc}[1]{{\mathbf{1}_{\left\{{#1}\right\}}}}

\def\innergetnumber#1[#2]#3{#2}
\def\getnumber{\expandafter\innergetnumber\jobname}

\newcommand{\bbJ}{{\mathbb{J}}}

\newcommand{\bbS}{{\mathbb{S}}}

\newcommand{\bfs}{{\mathbf{s}}}

\newcommand{\bfx}{{\mathbf{x}}}

\begin{document}
\title{Best Arm Identification in Batched Multi-armed Bandit Problems}
\author{Shengyu Cao\footnote{S.C.~is at School of Information Management and Engineering, Shanghai University of Finance and Economics.} \and Simai He\footnote{S.H.~is at School of Information Management and Engineering, Shanghai University of Finance and Economics.} \and Ruoqing Jiang\footnote{R.J.~is at Department of Economics, Columbia University.} \and Jin Xu\footnote{J.X.~is at School of Information Management and Engineering, Shanghai University of Finance and Economics.} \and Hongsong Yuan\footnote{H.Y.~is at School of Information Management and Engineering, Shanghai University of Finance and Economics. E-mail: yuan.hongsong@shufe.edu.cn.}}
\maketitle

\begin{abstract}
Recently multi-armed bandit problem arises in many real-life scenarios where arms must be sampled in batches, due to limited time the agent can wait for the feedback. Such applications include biological experimentation and online marketing. The problem is further complicated when the number of arms is large and the number of batches is small. We consider pure exploration in a batched multi-armed bandit problem. We introduce a general linear programming framework that can incorporate objectives of different theoretical settings in best arm identification. The linear program leads to a two-stage algorithm that can achieve good theoretical properties. We demonstrate by numerical studies that the algorithm also has good performance compared to certain UCB-type or Thompson sampling methods.

\end{abstract}



\section{Introduction} \label{sec:intro}
In a classical multi-arm bandit (MAB) problem \citep{LR85}, the learner chooses an action (or arm) at each time period in order to get the maximum expected cumulative reward over all periods. There is a trade-off between exploration, where the learner gather information about the expected rewards of arms, and exploitation, where the learner makes the best decision given current information. A variant of the classical MAB problem is best arm identification (BAI) (or pure exploration), in which the learner is required to recommend an arm with expected reward as high as possible \citep{bubeck2009pure, audibert2010best}. 

Consider an environment with $K$ arms, and the $j$th arm generates rewards according to a population $\Pi_j$ with expectation $\mu_j$. Hereinafter, we denote $\bmu=(\mu_1, \dots, \mu_K)^\top$, and let $\mu_{(1)}\ge \cdots \ge \mu_{(K)}$ be the order statistics of $\mu_j$'s. Let $j^*=\argmax_{1\le j\le K}\mu_j$ be the index of the best arm, and $\mu^*=\max_{1\le j\le K}\mu_j$ be its expected reward. Suppose the recommended arm is $\wh{j}$, then one is interested in the {\it probability of recommending the best arm} $\PB:=\Prob(\mu_{\wh{j}}=\mu^*)$, as well as the {\it simple regret} $\SR:=\mu^*-\mu_{\wh{j}}$. One other important measure is the total number of pulls $T$, which we also refer to as the {\it total sampling cost}. Apparently $\PB$, $\SR$ and $T$ may depend on $\bmu$, but we ignore such dependency in the notation. BAI is studied in several different theoretical frameworks. A common one is the $(\epsilon, \delta)-\mathrm{PAC}$ setting, which aims to find an $\epsilon$-optimal arm (namely an arm with $\mu_{\wh{j}}\ge \mu^*-\epsilon$) with probability at least $1-\delta$. Another popular direction is to pursue simple regret minimization (SRM). Two more prevalent settings are the fixed-confidence (FC) and the fixed-budget (FB) settings. In the FC setting, the learner minimizes $T$ under the constraint $\PB\ge 1-\delta$ for a given parameter $\delta$. In the FB setting, the learner maximizes $\PB$ under the budget constraint $T\le T_0$ for some given constant $T_0$. 

Recently, much attention is drawn to the scenario in which arms must be sampled 
in batches \citep{perchet2016batched}. In such a scenario, the learner is not 
allowed to make decisions within a batch, and the choice of arms for the next 
batch depends on the result of previous batches. Batched bandit problem arises 
in many applications. For instance, in social network platforms, one may be 
interested in finding out which user is the most active in a certain topic of 
interest. When a specific event on this topic occurs, analysts can follow a 
batch of users and record their activities via API. The most 
active user can be recommended after several events \citep{jun2016top}. Another area in which 
batched bandit problem emerges is clinical trials, where the experimenter 
chooses a set of treatments to test in each round of trial, and recommends the most effective treatment at the end of experiment. Other applications include online marketing and crowdsourcing \citep{esfandiari2021regret}. A common type of batch constraint is the $(m, \ell)$-batch constraint, in which the total number of pulls in a batch cannot exceed a constant $m$, and any single arm can be pulled at most $\ell$ times in a batch. In this paper, we restrict our study to the $(K, 1)$-batch setting.

In some of these real-world applications, the number of arms $K$ can be large, 
and the number of batches (or rounds) $R$ is relatively small. In the 
aforementioned social network example, there can be millions of users to follow 
and only a few events to occur due to time constraint for the experiment. With relative few samples one can observe from a single arm, the accuracy of BAI in batched bandits becomes a new challenge. Another challenging problem is how to control the total sampling cost when $K$ is large.

Traditional methods for MAB such as UCB and Thompson sampling have been 
generalized to solve batched bandit problems. The majority of existing algorithms are based on {\it peer-dependent} policies. That is, the decision of any single arm 
depends on the reward history of other arms. For the relatively poor performing 
arms, these algorithms tend to keep them as a future choice because 
there are chances that their expected rewards are high. Therefore, the learner 
is conservative in eliminating arms (eliminating an arm means to stop pulling the arm till 
the end of the time horizon) in early stages, unless there is a strong belief 
of their suboptimality. When $K$ is large and $R$ is relatively small, this strategy 
encounters two difficulties. First, it admits high proportions of arms in early 
batches, leading to high sampling cost. Second, the accuracy of BAI is seriously affected by the fact that $R$ is small. For one example, with $R$ rounds of {\it uniform exploration} for 
1-subgaussian bandits, $1-\PB$ is only upper bounded by $K \exp\left(-R \Delta_{2}^2/4 \right)$, 
where $\Delta_j=\mu^*-\mu_{(j)}$ is the suboptimality gap between the best and 
$j$th best arms \citep{lattimore2020bandit}. This upper bound is at least $K \exp\left\{-O(R) \right\}$ if $\Delta_2=O(1)$. For another example, with the same total sampling cost 
of $KR$, the successive rejects algorithm by \cite{audibert2010best} has an 
upper bound $\exp\left[-KR/\left\{\log(2K)H_2\right\}\right]$, where $H_2=\max_{2\le 
j\le K} j\Delta_{j}^{-2}$. Since $H_2 \gtrsim K$ when $\Delta_j$'s have a uniform bound, this upper bound is at 
least $\exp\left[-O\left\{R/\log(2K)\right\}\right]$. These bounds are hardly useful when 
$R\lesssim \log K$.

The main objective of this paper is to develop efficient algorithms for BAI in batched bandits when $K$ is large. Our way to address the difficulties mentioned in the preceding paragraph is to adopt a more aggressive arm eliminating scheme. When $K$ is large, the $\mu_j$'s are densely distributed if they are contained in a bounded range. Then the consequence of eliminating good arms is less severe, because there are many other arms that are almost equally good. On the other hand, the majority of arms are far from optimal, hence eliminating them early can greatly reduce overall sampling cost. This fact is more conspicuous under a Bayesian framework, where $\mu_1, \dots, \mu_K$ are generated independently according to some prior distribution. Furthermore, when $R$ is small, it is too demanding to identify the exact best arm with a very high accuracy. In this case, one should set a more realistic goal of recommending a good arm with low sampling cost.

We propose a two-stage algorithm LP2S, that can quickly eliminate arms in its first stage. The key feature of this stage is adopting a {\it peer-independent} policy, in which the decision of each arm is based only on its own reward history. The policy is derived from a linear program (LP) that can be regarded as a relaxation of an optimization problem for deriving peer-dependent policies. With a low sampling cost, the expected number of arms can be decreased to some $L\ll K$ after stage 1. In stage 2, we simply apply uniform exploration among all surviving arms. It is hoped that the sacrifice in optimality in stage 1 is small, and overall we can recommend a sufficiently good arm at the end of stage 2. 

\smallskip

\noindent \textbf{Contributions.} The main contributions of this paper are as follows.
\begin{enumerate}
	\item[] 1) We establish a universal optimization framework, that can include many existing methods, for deriving peer-dependent policies. The optimization framework is further relaxed to an LP framework that can generate peer-independent policies. Building such an LP framework and establishing connections between peer-dependent and peer-independent policies is unprecedented in the literature. 
	\item[] 2) Under the introduced LP framework, we specify four different settings, namely PAC, SRM, FC and FB. The PAC ans SRM settings are particularly suitable for the large $K$ small $R$ case.
	\item[] 3) We propose a two-stage algorithm LP2S, which has good theoretical properties and strong numerical performances. The first stage of the algorithm is a peer-independent procedure induced from LP, and can efficiently select good arms by setting thresholds on cumulative rewards of individual arms. The two-stage algorithm incurs an expected total sampling cost of $O(L h(R))$, where $h(R)$ is at most a polynomial of $R$. We also obtain reasonable upper bounds for different variants of LP.
\end{enumerate}

\subsection{Related Work}

\textbf{BAI.} BAI without batch constraints has been studied extensively under 
different theoretical frameworks in the last few years. 
\begin{itemize}
	\item Under the $(\epsilon, \delta)-\mathrm{PAC}$ setting, a lower bound of 
	$\Omega\left(K\log\{O(1/\delta)\}/\epsilon^{2}\right)$ for the sampling 
	cost is obtained by \cite{evendar2002pacbounds} and 
	\cite{mannor2004sample}, and median elimination 
	\citep{evendar2002pacbounds} is proposed to achieve this lower bound. 
	\cite{kalyanakrishnan2012pac} propose the LUCB algorithm for the top-$k$ 
	arm recommendation problem under $(\epsilon, \delta)-\mathrm{PAC}$ setting.
	
	\item In simple regret analysis, studies include \cite{LLR16} who obtain 
	minimax regret bound, \cite{wuthrich2021regret} who establish regret bounds 
	for variants of expected improvement and UCB algorithms, and 
	\cite{zhao2023revisiting} who establish bounds for sequential halving 
	algorithm. 
	
	\item For the FC setting, existing algorithms include exponential-gap elimination \citep{KKS13}, LUCB \citep{jamieson2014fixedconfidence}, lil'UCB \citep{JMNB14} and so on. The lower bound for the sampling cost is shown to be $(2-4\delta)\sum_{j=2}^K \Delta_j^{-2} \log\log \Delta_j^{-2}$ by \cite{JMNB14}. Furthermore, \cite{KCG16} propose key measures to represent the complexity of the problem in both FB and FC settings. 
	
	\item For the FB setting, algorithms such as successive rejects \citep{audibert2010best}, sequential halving \citep{KKS13}, and sequential elimination \citep{SNT17} have been proposed. \cite{carpentier2016tight} show that the fundamental lower bound for $1-\PB$ is $\exp\left\{-T/(H\log K)\right\}$ with $H=\sum_{j=2}^K \Delta_j^{-2}$.

\end{itemize}
Although our paper is not devoted to any of the four settings, our method can be tailored to reflect the objective of any one of the four. Furthermore, we note that some of the aforementioned papers already reveal features of batch pulling. For instance, the exponential-gap elimination algorithm and sequential halving algorithm in \cite{KKS13} both include steps that sample a group of arms multiple times before the next decision. Nonetheless, their studies do not center around the batch constraints, especially when $R\ll K$.

\noindent \textbf{Batched bandits.} The batched bandit problem is first studied under the classical exploration-exploitation scheme. \cite{perchet2016batched} use an explore-then-commit policy 
for a two-arm bandit with sub-Gaussian rewards, and obtain upper and lower bounds for the cumulative regret . \cite{perchet2016batched}, along with studies such as \cite{jin2021anytime, jin2021double} assume a static grid in which the batch sizes need to be pre-determined, whereas several other papers \citep{gao2019batched, esfandiari2021regret} allow adaptive grids. A fundamental question is how many batches are needed to achieve the optimal cumulative regret. \cite{gao2019batched} show that this number is $\Omega(\log \log T)$ for minimax optimality, and $\Omega(\log T/\log \log T)$ for problem-dependent optimality. Algorithms proposed in these papers all belong to the UCB category. On the other hand, \cite{kalkanli2021abatched, kalkanli2023asymptotic} and \cite{karbasi2021thompson} introduce the batched Thompson sampling. 

Studies that focus solely on BAI in batched bandits are relatively scarce. \cite{jun2016top} propose the BatchRacing algorithm for the FC setting, and the BatchSAR algorithm for the FB setting. They show that the sampling cost has an extra $\log K$ term compared to the non-batch scenario. \cite{agarwal2017limit} consider the FC setting and propose an aggressive elimination algorithm that only requires $\Theta(\log^* K)$ batches. \cite{komiyama2021optimal} propose a two-stage exploration algorithm and carry out simple regret analysis under Bayesian setting. In Table \ref{tab:literature}, we provide a summary of the aforementioned papers in batched bandits as well as our paper. In terms of basic settings, \cite{komiyama2021optimal} is closest to our study. 

\begin{table}[h]
	\resizebox{\textwidth}{!}{
		\begin{tabular}{|ccccc|}
			\hline
			\multirow{2}{*}{\textbf{Article}} &
			\multirow{2}{*}{\textbf{BAI}} & \textbf{Batch} &
			\textbf{Bayesian} & \multirow{2}{*}{\textbf{Algorithm}} \\
			&  & \textbf{grid} & \textbf{bandit} & \\ 
			\hline
			\cite{perchet2016batched} & no & static & no 
			& UCB type \\ \hline
			\cite{gao2019batched} & no & static/adaptive & no & UCB type \\ \hline
			\cite{jin2021anytime, jin2021double} & no & static & no &  UCB type \\ \hline
			\cite{esfandiari2021regret} & no & adaptive & no &  UCB type \\ \hline 
			\cite{kalkanli2021abatched, kalkanli2023asymptotic} & no & adaptive & yes & Thompson \\ \hline 
			\cite{karbasi2021thompson} & no & adaptive & yes & Thompson \\ \hline 
			\cite{jun2016top} & yes & adaptive & no &
			UCB type \\ \hline
			\cite{agarwal2017limit} & yes & adaptive & no & sample mean based \\ \hline
			\cite{komiyama2021optimal} & yes & adaptive & yes & UCB type \\ \hline 
			This paper & yes & adaptive & yes & novel \\ \hline 
	\end{tabular}}
	\caption{Summary of literature on batched bandit problem.}
	\label{tab:literature}
\end{table}

\noindent \textbf{Bayesian bandits.} The analysis in this paper relies heavily 
on the Bayesian assumption. In Bayesian bandits, \cite{lai1987adaptive} 
first generalizes UCB method to Bayesian bandits, and obtains asymptotic lower bound for Bayesian cumulative regret. \cite{G79} introduces the seminal Gittins index for bandits with discounts. \cite{kaufmann2012bayesian} 
extend classical UCB to Bayes-UCB algorithm. \cite{russo2018learning} derive $O(1/\sqrt{T})$ bounds for Bayesian simple regret. \cite{russo2020simple} proposes several variants Thompson sampling. \cite{shang2020fixed} propose a sampling rule inspired by \cite{russo2020simple} in FC setting for bandits with Gaussian rewards.


\section{Optimization and Linear Programming Frameworks} \label{sec:opt-lp}
\subsection{Problem Setup} \label{subsec:setup}
For simplicity, we consider Bernoulli stochastic bandits where the population $\Pi_j$ is Bernoulli$(\mu_j)$, although our method extends naturally to discrete populations with finite support. We restrict our study to the $(K,1)$-batch setting, and constrain the number of batches to be a pre-determined number $2R$. We assume $K\to \infty$, $R\to \infty$, but do not specify the relative order of $R$ and $K$ for now. In the $r$th batch, the learner decides the set of arms $\varphi_r=\left\{\wh{j}_{r,1},\dots, \wh{j}_{r,b_r}\right\} \subset \{1, \dots, K\}$ to pull based on the history of all decisions and rewards, and gets rewards $\bfx_r=\left(x_{r,\wh{j}_{r,1}}, \dots, x_{r, \wh{j}_{r,b_r}}\right)^\top$. The total number of pulls is $T=\sum_{r=1}^{2R} b_r$.

We study the problem under a Bayesian setting, in which $\mu_j$'s are sampled independently from some prior $\pi$ with finite moments up to order $R$. We define {\it Bayesian probability of recommending the best arm}
\begin{align*}
	\BPB=\Expect^{\bmu} (\PB),
\end{align*}
and {\it Bayesian simple regret}
\begin{align*}
	\BSR=\Expect^{\bmu} (\SR),
\end{align*}
where $\Expect^{\bmu}$ is the expectation over the joint distribution of $\bmu$.

\subsection{An Optimization Framework for Peer-dependent Policies} \label{subsec:opt}
In Bernoulli bandits, assume any fixed single arm (we name it the {\it focal arm} hereinafter) with expected reward $\mu$ has generated rewards $(x_1, \dots, x_r)$ up to round $r$.  Since the cumulative reward $s=\sum_{t=1}^r x_{t}$ is a sufficient statistic for $\mu$, the learner only needs to record $s$, rather than the whole reward path $(x_1, \dots, x_r)$. We introduce an additional state $F$ to denote that the arm is already eliminated up to round $r$. We assume that, if the arm is in state $F$, it is no longer a candidate in the remaining rounds. This results in a state space $\bbS_r=\left\{0,1,\dots, r, F\right\}$. To consider the states of other arms, let
\begin{align*}
	\bbS_{r}^{-}=\left\{\bfs_{r}^{-}=(s_1,\dots, s_{K-1})\in \bbS_r^{K-1}. 
	\right\}
\end{align*}

A {\it peer-dependent policy} is one in which actions for a single arm depends on outputs of other arms. In a peer-dependent policy, the action of the focal arm in round $r+1$ is a function $\wt{a}(r,s,\bfs_{r}^{-}) \in [0,1]$, representing the probability of pulling the arm. Most existing methods for batched bandit problems rely on peer-dependent policies, and many align with explicit forms of $\wt{a}(r,s,\bfs_{r}^{-})$. A few examples are given in the following.

\noindent {\bf Example 1.} In the median elimination algorithm \citep{evendar2002pacbounds}, the elimination rule corresponds to
\begin{align*}
	\wt{a}(r,s,\bfs_{r}^{-}) = \begin{cases}
		1, & \mathrm{\ if\ } s \geq \mathrm{median}\left\{s_1, \dots, s_K\right\}, \\
		0, & \mathrm {\ otherwise}.	
	\end{cases}	
\end{align*}
for elimination rounds, and $\wt{a}(r,s,\bfs_{r}^{-})=1$ for rounds without elimination.

\noindent {\bf Example 2.} In the BatchRacing algorithm \citep{jun2016top}, the upper and lower confidence bounds (UCBs and LCBs) of all candidate arms are computed at first
\begin{align*}
	U=s/r+ D\left(r,\sqrt{\delta/(6K)}\right), & \quad L=s/r- D\left(r,\sqrt{\delta/(6K)}\right), \\
	U_j=s_j/r+ D\left(r,\sqrt{\delta/(6K)}\right), & \quad L_j=s_j/r- D\left(r,\sqrt{\delta/(6K)}\right), \mathrm{\ for\ } s_j\neq F. 
\end{align*}
Then the acceptance-rejection rule renders
\begin{align*}
	\wt{a}(r,s,\bfs_{r}^{-}) = \begin{cases}
		0, & \mathrm{\ if \ } L>\max_{s_j\neq F} U_j \mathrm{\ (acceptance)}, \\
		0, & \mathrm{\ if \ } U< \max_{s_j \neq F} L_j \mathrm{\ (rejection)} \\
		1, & \mathrm {\ otherwise}.
	\end{cases}
\end{align*}

\noindent {\bf Example 3.} In the two stage exploration algorithm \citep{komiyama2021optimal}, the UCBs and LCBs arms are computed in the last round of the first stage (i.e. round $r=qT/K$)
\begin{align*}
	L_{j}=\frac{s_j}{qT/K} - \sqrt{\frac{K\log T}{qT}}, \quad U_{j}=\frac{s_j}{qT/K}+\sqrt{\frac{K\log T}{qT}}.
\end{align*}
Then decision rule is in line with
\begin{align*}
	\wt{a}(r,s,\bfs_{r}^{-}) = \begin{cases}
		1, & \mathrm{\ if \ }U \geq \max_{s_j\neq F} L_{j}, \\
		0, & \mathrm {\ otherwise}
	\end{cases}
\end{align*}
for $r=qT/K$. For $r<qT/K$ and $r>qT/K$, the rule is simply $\wt{a}(r,s,\bfs_{r}^{-})=1$. 

\medskip 

For any peer-dependent policy, define for round $0\le r\le R$
\begin{align*}
	\wt{P}(r,s,\bfs_{r}^{-}) & =\Prob\left(\{\mathrm{arms\ in\ state\ } (s,\bfs_{r}^{-})\}\right), \\
	\wt{P}_1(r,s,\bfs_{r}^{-}) & =\Prob\left(\{\mathrm{arms\ in\ state\ } (s,\bfs_{r}^{-})\}\cap \{x_r=1\} \right), \\
	\wt{P}_0(r,s,\bfs_{r}^{-}) & =\Prob\left(\{\mathrm{arms\ in\ state\ } (s,\bfs_{r}^{-})\}\cap \{x_r=0\} \right).
\end{align*}
Clearly, the three probabilities satisfy
\begin{align}
	\wt{P}(r,s,\bfs_{r}^{-})=\wt{P}_1(r,s,\bfs_{r}^{-})+\wt{P}_0(r,s,\bfs_{r}^{-}). \label{eq:P0P1_sum_dependent}
\end{align}
For any $0\le s\le r$ and $\bfs_{r+1}^{-}\in \bbS_{r+1}^{-}$, we have
\begin{align}
	& \wt{P}_1\left(r+1,s+1,\bfs_{r+1}^{-}\right)  = 
	\sum_{\bfs_{r}^{-}\in \bbS_{r}^{-}} \wt{P}\left(r,s,\bfs_{r}^{-}\right) 
	\wt{a} \left(r,s,\bfs_{r}^{-}\right) q(r,s) Q\left(\bfs_{r}^{-}, 
	\bfs_{r+1}^{-}\right), \label{eq:P1_recursion_dependent}\\
	& \wt{P}_0\left(r+1,s,\bfs_{r+1}^{-}\right)  = 
	\sum_{\bfs_{r}^{-}\in \bbS_{r}^{-}} \wt{P}\left(r,s,\bfs_{r}^{-}\right) 
	\wt{a} \left(r,s,\bfs_{r}^{-}\right) \left\{1-q(r,s)\right\} 
	Q\left(\bfs_{r}^{-}, \bfs_{r+1}^{-}\right) \label{eq:P0_recursion_dependent},
\end{align}
where $q(r,s)=\Expect\left\{\Prob(x_{r}=1\mid r,s)\right\}=\Expect(\mu\mid 
r,s)$ is the posterior mean of $\mu$ given $s$ successes in $r$ pulls, and 
$Q\left(\bfs_{r}^{-}, 
\bfs_{r+1}^{-}\right)= \Expect \left\{\Prob\left(\bfs_{r}^{-}\to 
\bfs_{r+1}^{-} \mid \bfs_{r}^{-} \right)\right\}$ is the posterior mean of the probability of transition $\bfs_{r}^{-}\to \bfs_{r+1}^{-}$. To explain \eqref{eq:P1_recursion_dependent}, note that the transition from state $\left(s,\bfs_{r}^{-}\right)$ in round $r$ to $\left(s+1,\bfs_{r+1}^{-}\right)$ with $x_{r+1}=1$ in round $r+1$ requires three independent events: the focal arm is pulled, its instant reward is 1, and the other arms transit from $\bfs_{r}^{-}$ to $\bfs_{r+1}^{-}$. The terms $\wt{a}\left(r,s,\bfs_{r}^{-}\right)$, $q(r,s)$ and $Q\left(\bfs_{r}^{-}, \bfs_{r+1}^{-}\right)$ are the probabilities of the three events. Similar explanation holds for \eqref{eq:P0_recursion_dependent}.

Any problem that admits peer-dependent policies can be formulated as 
\begin{align}
	\min & \quad f\left(\left\{\wt{P}\left(r,s,\bfs_{r}^{-}\right), 0\le r\le R, 0\le s\le r, \bfs_{r}^{-}\in \bbS_{r}^{-}\right\} \right) \nonumber \\
	\mathrm{s.t.} & \ \eqref{eq:P0P1_sum_dependent}, \eqref{eq:P1_recursion_dependent} \mathrm{\ and \ } 
	\eqref{eq:P0_recursion_dependent} \mathrm{\ hold,} \nonumber \\
	& \wt{P}_1\left(0,0,\mathbf{0}\right) = 1, \quad \wt{P}_1\left(r+1,0,\bfs_{r}^{-}\right) = 0, \quad \wt{P}_0\left(r,r,\bfs_{r}^{-}\right) = 0, \quad 0\le r\le R, \label{eq:boundary_constraints_dependent} \\
	& 0\le a\left(r,s,\bfs_{r}^{-}\right)\le 1, \label{eq:action_constraint_dependent} \\
	& \mathrm{Additional\ problem-specific\ constraints,} \label{eq:additional_constraint_dependent}
\end{align}
where $f$ is a function expressing the objective of the problem, \eqref{eq:boundary_constraints_dependent} are natural boundary constraints, and \eqref{eq:additional_constraint_dependent} are constraints customized to specific purposes in the problem. We use {\bf OPT-dep} to denote this optimization problem.

\subsection{An LP framework for peer-independent policies} \label{subsec:lp}
Apparently, OPT-dep is {\it intractable} because of its exploding state space. In Bayesian bandits however, the distributions of the outputs of non-focal arms can be inferred without observing their realized rewards. In particular, if a non-focal arm $j$ is pulled $r$ times, then the marginal distribution of its cumulative reward $s_{j}$ can be obtained since we know the prior distribution of $\mu_j$. Then, any action of the focal arm that depends on $s_j$ can be predicted in distribution. In other words, not much information is lost by ignoring the outputs of non-focal arms.

Following this idea, we try to relax OPT-dep to an optimization problem that only relies on the states of the focal arm. For any feasible solution of OPT-dep, consider aggregating variables
\begin{align*}
	P(r,s) & =\sum_{\bfs_{r}^{-}\in \bbS_{r}^{-}}  
	\wt{P} \left(r,s,\bfs_{r}^{-}\right) = \Prob\left(\{\mathrm{focal\ arm\ in\ 
		state\ } s \mathrm{\ in\ round \ } r\}\right), \\
	P_i(r,s) & =\sum_{\bfs_{r}^{-}\in \bbS_{r}^{-}}  
	\wt{P}_i \left(r,s,\bfs_{r}^{-}\right) =\Prob\left(\{\mathrm{focal\ arm\ 
		in\ state\ } s \mathrm{\ in\ round \ } r \}\cap \{x_r=i\} \right), \quad i=0,1.
\end{align*}
Then we immediately have
\begin{align}
	 P(r,s) = P_1(r,s) + P_0(r,s), \quad 0\le r\le R, 0\le s\le r. \label{eq:P0P1_sum} 
\end{align}
From \eqref{eq:P1_recursion_dependent} we get 
\begin{align*}
	P_1(r+1,s+1) & =\sum_{\bfs_{r+1}^{-}\in \bbS_{r+1}^{-}}  
	\wt{P}_1 \left(r+1,s+1,\bfs_{r+1}^{-}\right), \\
	& = \sum_{\bfs_{r+1}^{-}\in \bbS_{r+1}^{-}} \sum_{\bfs_{r}^{-}\in 
		\bbS_{r}^{-}} \wt{P}\left(r,s,\bfs_{r}^{-}\right) 
	\wt{a}\left(r,s,\bfs_{r}^{-}\right) q(r,s) Q\left(\bfs_{r}^{-}, 
	\bfs_{r+1}^{-}\right), \\
	& = q(r,s) \sum_{\bfs_{r}^{-}\in 
		\bbS_{r}^{-}} \wt{P}\left(r,s,\bfs_{r}^{-}\right) 
	\wt{a}\left(r,s,\bfs_{r}^{-}\right) \sum_{\bfs_{r+1}^{-}\in \bbS_{r+1}^{-}} 
	Q\left(\bfs_{r}^{-}, 
	\bfs_{r+1}^{-}\right).
\end{align*}
Since $\sum_{\bfs_{r+1}^{-}\in \bbS_{r+1}^{-}} Q\left(\bfs_{r}^{-}, \bfs_{r+1}^{-}\right)= \Expect \sum_{\bfs_{r+1}^{-}\in \bbS_{r+1}^{-}} \Prob\left(\bfs_{r}^{-}\to \bfs_{r+1}^{-} \mid \bfs_{r}^{-} \right)=1$, we get
\begin{align}
	P_1(r+1,s+1) = q(r,s) \sum_{\bfs_{r}^{-}\in 
		\bbS_{r}^{-}} \wt{P}\left(r,s,\bfs_{r}^{-}\right) 
	\wt{a}\left(r,s,\bfs_{r}^{-}\right). \label{eq:P1_recursion_middle}
\end{align}
By the same argument, we have
\begin{align}
	P_0(r+1,s) = \left\{1-q(r,s)\right\} \sum_{\bfs_{r}^{-}\in 
		\bbS_{r}^{-}} \wt{P}\left(r,s,\bfs_{r}^{-}\right) 
	\wt{a}\left(r,s,\bfs_{r}^{-}\right). \label{eq:P0_recursion_middle}
\end{align}
From \eqref{eq:P1_recursion_middle} and \eqref{eq:P0_recursion_middle}, we 
get
\begin{align}
	\frac{P_1(r+1, s+1)}{q(r,s)} = \frac{P_0(r+1, s)}{1-q(r,s)}. 
	\label{eq:P_recursion_middle}
\end{align}
Since $0\le \wt{a}\left(r,s,\bfs_{r}^{-}\right)\le 1$, we further obtain from 
\eqref{eq:P1_recursion_middle} and \eqref{eq:P0_recursion_middle}
\begin{align}
	P_1(r+1,s+1) & \le q(r,s) \sum_{\bfs_{r}^{-}\in \bbS_{r}^{-}} 
	P\left(r,s,\bfs_{r}^{-}\right) =q(r,s) P(r,s), 
	\label{eq:P1_recursion_middle2} \\
	P_0(r+1,s) & \le \left\{1-q(r,s)\right\} \sum_{\bfs_{r}^{-}\in \bbS_{r}^{-}} 
	P\left(r,s,\bfs_{r}^{-}\right) = \left\{1-q(r,s)\right\} P(r,s). 
	\label{eq:P0_recursion_middle2}
\end{align}
Combining \eqref{eq:P_recursion_middle}, \eqref{eq:P1_recursion_middle2} and 
\eqref{eq:P0_recursion_middle2}, we have
\begin{align}
	\frac{P_1 (r+1, s+1)}{q(r,s)} = \frac{P_0(r+1,s)}{1-q(r,s)} \le P(r,s), \quad 0\le r\le 
	R-1, 0\le s\le r.
	\label{eq:P_recursion}
\end{align}
Let 
$$a(r,s)= \frac{P_1 (r+1, s+1)}{q(r,s)P(r,s)} = \frac{P_0(r+1,s)}{\left\{1-q(r,s)\right\} P(r,s)}\in 
[0,1]$$ 
be the action of the focal arm in 
round $r+1$. Then the set of actions $A=\{a(r,s): 0\le r\le R, 0\le s\le r\}$ 
defines a {\it peer-independent policy}, where actions for a single arm do not 
depend on outputs of any other arm.

Any problem that admits peer-independent policies can be formulated as 
\begin{align}
	\min & \quad f\left(\left\{P(r,s), 0\le r\le R, 0\le s\le r\right\} \right) \nonumber \\ 
	\mathrm{s.t.} & \ \eqref{eq:P0P1_sum} \mathrm{\ and\ } \eqref{eq:P_recursion} \mathrm{\ hold,}  \label{eq:first_constraint} \\
	& P_1(0,0)=1, \quad P_1(r+1,0)=0, \quad P_0(r,r)=0, \quad 0\le r\le R, \label{eq:boundary_constraints} \\  
	& \mathrm{Additional\ problem-specific\ constraints.} \label{eq:additional_constraint_independent}
\end{align}
We use {\bf OPT-ind} to denote this optimization problem. Figure 
\ref{fig:recursion} demonstrates a binomial tree that represents the recursion of probabilities in OPT-ind.

\begin{figure}[h!]
	\centering
	\scalebox{0.45}{\includegraphics{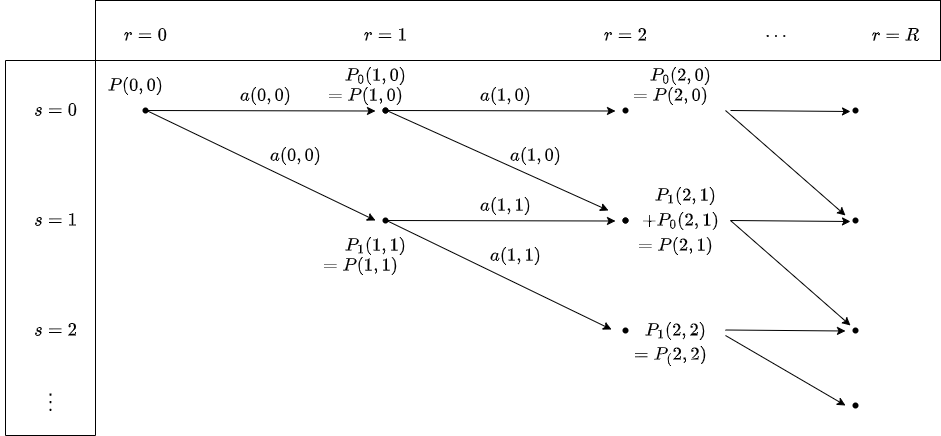}}
	\caption{Recursion of probabilities in OPT-ind.}
	\label{fig:recursion}
\end{figure}

We consider special a case of OPT-ind. Assume $f$ is linear in $P(r,s), 0\le 
r\le R, 0\le s\le r$. Also fix the following two additional constraints
\begin{itemize}
	\item Additional constraint 1:
	\begin{align}
		\sum_{s=0}^R P(R,s)=L/K, \label{eq:survive_constraint}
	\end{align}
	where $L$ is a given number. This constraint implies that the probability of an arm surviving to the last round is $L/K$.
	
	\item Additional constraint 2:
	\begin{align}
		\sum_{s=0}^R w(s) P(R,s) \ge (1-\delta_0) \sum_{s=0}^R P(R,s), \label{eq:conditional_constraint}
	\end{align}
	where $\delta_0$ is a given constant in $[0,1]$, and $w(s)$ is a known function 
	non-decreasing in $s$. When $\delta_0$ is small, this constraint assures that the end probabilities $P(R,s)$ are tilted towards large $s$ values, compared to the standard binomial probabilities $\Expect^{\mu} \Prob(S=s)$ with $S\sim \mathrm{Binomial}(R,\mu)$.
\end{itemize}
Then OPT-ind becomes a linear program, which we refer to as {\bf LP-ind}.

In this paper, we mainly consider three variants of LP-ind. Namely,
\begin{enumerate}
	\item PAC setting, where we take 
	\begin{align*}
		f & = \sum_{r=1}^R \sum_{s=0}^r P(r,s), \\ 
		w(s) & =\Prob(\mu\ge \mu_0\mid R,s),  
	\end{align*}
	for some given $\mu_0$. Note that for each $r$, $\sum_{s=0}^r P(r,s)$ is the probability of surviving till the end of round $r$, hence $f$ is the expected sampling cost for one arm. $w(s)$ is the posterior probability of an arm having expected reward $\ge \mu_0$, given that it generated $s$ successes in $R$ pulls. With this choice of $w(s)$, an arm that survives till the end of round $R$ should at least be an $(1-\mu_0)$-optimal arm. We denote this problem {\bf LP-PAC}. When $R \ll K$, it is difficult to find the exact best arm, so LP-PAC can be used to find good arms efficiently.
	
	\item SRM setting, where we choose
	\begin{align*}
		f & = \sum_{r=1}^R \sum_{s=0}^r P(r,s), \\
		w(s) & = \Expect \mu^*- \Expect\left(\mu \mid R,s\right).
	\end{align*}
	This choice of $w(s)$ matches the goal of keeping the simple regret as low as possible. We denote this problem {\bf LP-SRM}.
	
	\item FC setting, where we use 
	\begin{align*}
		f & =\sum_{r=1}^R \sum_{s=0}^r P(r,s) \\
		w(s) & = \Prob(\mu=\mu^*\mid R,s).
	\end{align*}
	We denote this problem {\bf LP-FC}. Since enough observations must be made in order to find the best arm with good probability,  LP-FC is typically suitable for the case $R \gtrsim K$.
\end{enumerate}
One can also consider the FB setting, where $f =-\sum_{s=0}^R P(R,s) \Prob\left\{\mu=\mu^*\mid R,s \right\}$, $w(s) \equiv 1$, and exert an additional constraint $\sum_{r=1}^R \sum_{s=0}^r P(r,s)\le T_0$. However, we do not include this problem in our study for the sake of maintaining the analysis in the same framework.

\subsection{An LP-induced two-stage algorithm}\label{subsec:algo}
Let $\left\{P^*(r,s), P_1^*(r,s), P_0^*(r,s): 0\le r\le R, 0\le s\le r 
\right\}$ be an optimal solution of LP-ind. Then it induces a set of actions
\begin{align}
	a^*(r,s)= \frac{P_1^*(r+1,s+1)}{q(r,s) P^*(r,s)} = 
	\frac{P_0^*(r+1,s)}{\left\{1-q(r,s)\right\} P^*(r,s)}, \quad 0\le r\le R-1, 0\le s\le r. \label{eq:optimal_action}
\end{align}
We now propose LP2S, a two-stage algorithm consisting of an elimination stage based on the LP-induced actions $\left\{a^*(r,s): 0\le r\le R, 0\le s\le r \right\}$, and a finer exploration stage. The two stages are described by Algorithms \ref{alg:stage1} and \ref{alg:stage2} respectively.

\begin{algorithm}[H]
	\begin{algorithmic}[1]
		\STATE Initialization: $\bbJ_0=\{1,\dots,K\}$, $s_j=0$ for $j\in 
		\bbJ_0$.
		\FOR{$r=1$ to $R$}
		\FOR{ each $j \in \bbJ_{r-1}$}
		\STATE Generate an independent $z_{r,j}\sim 
		\mathrm{Bernoulli}\left\{a^*(r,s_j)\right\}$.
		\IF{$z_{r,j}=1$}
		\STATE Pull arm $j$, get reward $x_{r,j}$. 
		\STATE Update $s_j\leftarrow s_j+x_{r,j}$.
		\ENDIF
		\ENDFOR
		\STATE Update $\bbJ_{r}=\{j\in \bbJ_{r-1}: z_{r,j}=1\}$.
		\ENDFOR
		\STATE Output: $\bbJ_{R}$, the set of surviving arms.
	\end{algorithmic}
	\caption{Stage 1: Implement optimal policy induced by LP-ind}
	\label{alg:stage1}
\end{algorithm}

\begin{algorithm}[H]
	\begin{algorithmic}[1]
		\STATE Run $R$ rounds of uniform exploration among $\bbJ_{R}$. 
		Output arm $\wh{j}$ with the highest cumulative reward.
	\end{algorithmic}
	\caption{Stage 2: Run uniform exploration for arms surviving from stage 1}
	\label{alg:stage2}
\end{algorithm}

In stage 1 (specified by Algorithm \ref{alg:stage1}), we implement actions $\left\{a^*(r,s): 0\le r\le R, 0\le s\le r \right\}$ for each arm separately. In particular, in round $r$, if arm $j$ is not eliminated yet and has cumulative reward $s_j$, we pull it with probability $a^*(r,s_j)$ and eliminate it with probability $1-a^*(r,s_j)$. We then update the cumulative rewards of all remaining arms and proceed to the next round. At the end of round $R$, we are left with a set $\bbJ_{R}$ of arms. Note that each arm survives independently with probability $L/K$ according to additional constraint 1 expressed by \eqref{eq:survive_constraint}, and $J=|\bbJ_{R}|$ follows a Binomial$(K,L/K)$ distribution.

If there are no arms surviving stage 1 (that is, $J=0$), our algorithm terminates and we recommend a random arm. If $J>0$, we simply run $R$ rounds of {\it uniform exploration} among $\bbJ_{R}$ (i.e. pulling each arm $R$ times) in stage 2, and recommend the arm with the highest cumulative reward in this round, as shown by Algorithm \ref{alg:stage2}.

With suitable parameters, stage 1 essentially carries out a rather aggressive elimination scheme to quickly filter out ``bad'' arms. The number of arms $J$ surviving stage 1 is random but its expected value is $L$. Stage 2 further explores among the ``good'' arms, with an expected sampling cost of $LR$. The parameter $L$ needs to be chosen carefully. If $L$ is too small, then the risk of no arms surviving stage 1 becomes significant. On the other hand,  a too large value of $L$ causes a high sampling cost in stage 2.


\section{Theoretical Results}  \label{sec:theory}
\subsection{Property of LP-ind Induced Policy} \label{subsec:threshold}
It is intuitive that the action $a^*(r,s)$ defined by \eqref{eq:optimal_action} should be non-decreasing in $s$. That is, arms with higher cumulative rewards should enjoy a high probability of getting pulled next. Theorem \ref{thm:LP_thresh} states that it is in fact a ``threshold'' policy.
\begin{theorem} \label{thm:LP_thresh}
	Suppose LP-ind is feasible. Then there exists an optimal solution 
	$$\left\{P^*(r,s), P_1^*(r,s), P_0^*(r,s): 0\le r\le R, 0\le s\le r 
	\right\}$$ 
	such that the actions $\left\{a^*(r,s): 0\le r\le R-1, 0\le s\le r \right\}$ defined by \eqref{eq:optimal_action} satisfy
	\begin{align*}
		a^*(r,s)= \begin{cases}
			0,\quad s<s^*(r) \\
			1, \quad s>s^*(r)
		\end{cases}
	\end{align*}
	with some $0\le s^*(r)\le r$ non-decreasing in $r$ for $0\le r\le R-1$.
\end{theorem}
The proof of Theorem \ref{thm:LP_thresh} can be found in Appendix \ref{sec:proof_LP_thresh}. According to Theorem \ref{thm:LP_thresh}, for each round $r$, there exists a threshold $s^*(r)$ such that the arm will be eliminated if the cumulative reward $s$ is under $s^*(r)$, and it will be pulled with probability 1 if the cumulative reward is above $s^*(r)$. With suitable choice of parameters (e.g. choosing a small $\delta_0$), $s^*(r)$ can be quite large, which results in a austere threshold policy that quickly eliminates arms. Figure \ref{fig:showcase} compares arm eliminating rates of the threshold policy induced by LP-ind and BatchRacing \citep{jun2016top} for an example with $K=100$. We can see that the threshold policy quickly eliminates most arms in early rounds, whereas BatchRacing keeps all arms in the first 120 rounds, and slowly eliminate arms in the subsequent rounds.

\begin{figure}[tbp]
	\label{fig:showcase}\centering{\includegraphics[width=0.45 
		\textwidth]{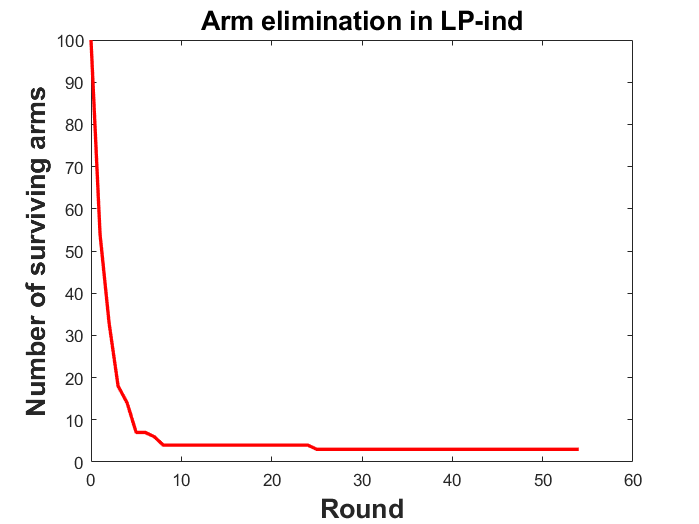}} {\includegraphics[width=0.45
		\textwidth]{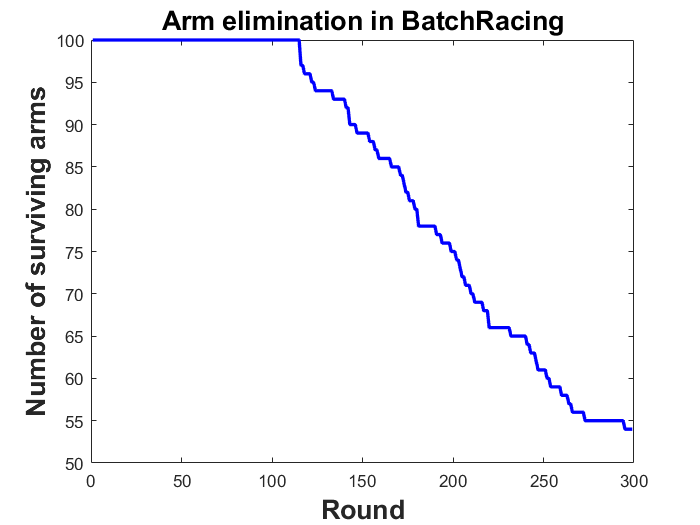}}
	\caption{Arm elimination in LP-ind and BatchRacing}
\end{figure}

\subsection{Total Sampling Cost} \label{subsec:cost}
Next, we concentrate on the total sampling cost. First of all, for LP-PAC/LP-SRM/LP-FC/, Theorem \ref{thm:LP_cost} provides an upper bound for its optimal value.

\begin{theorem} \label{thm:LP_cost}
	Suppose LP-PAC/LP-SRM/LP-FC is feasible. Then its optimal value $f^*$ satisfies
	$$f^* \le \frac{L}{K\Expect_{\pi}(\mu^R)}\sum_{r=1}^R 
	\Expect_{\pi}(\mu^r).$$ 
\end{theorem}

For beta priors, we have the following corollary.
\begin{coro}\label{coro:LP_cost}
	Assume $\pi=\mathrm{Beta}(a,b)$ for absolute constants $a,b$. Suppose LP-PAC/LP-SRM/LP-FC is feasible. Then 
	\begin{align*}
		f^*\lesssim \begin{cases}
			LR/K, & \mathrm{\ if\ } 0<b<1, \\
			LR\log R/K, & \mathrm{\ if\ } b=1, \\
			LR^b/K, & \mathrm{\ if\ } b>1.
		\end{cases}
	\end{align*}
\end{coro}
We provide the proofs of Theorem \ref{thm:LP_cost} and Corollary \ref{coro:LP_cost} in Appendix \ref{sec:proof_LP_cost}.

Since the expected sampling cost of stage 2 is $LR$, the total sampling cost $T$ of LP2S satisfies $\Expect(T) = Kf^* + LR$. The following corollary is then immediate.
\begin{coro}\label{coro:LP2S_cost} 
	Suppose LP-PAC/LP-SRM/LP-FC is feasible. The expected total sampling cost $T$ of the two-stage algorithm satisfies
	\begin{align*}
		\Expect(T) \lesssim \frac{L}{\Expect_{\pi}(\mu^R)}\sum_{r=1}^R 
		\Expect_{\pi}(\mu^r) + LR.
	\end{align*}
	If $\pi=\mathrm{Beta}(a,b)$, then
	\begin{align*}
		\Expect(T) \lesssim \begin{cases}
			LR, & \mathrm{\ if\ } 0<b<1, \\
			LR\log R, & \mathrm{\ if\ } b=1, \\
			LR^b, & \mathrm{\ if\ } b>1.
		\end{cases}
	\end{align*}
\end{coro}
Corollary \ref{coro:LP2S_cost} indicates that the overall expected sampling cost has an upper bound that does not rely on $K$. Under beta prior, the upper bound is only linear in $L$ and at most polynomial in $R$. When $R\ll K$ and $L\ll K$, the sampling cost should be much smaller than most existing methods.

\subsection{Confidence and Regret} \label{subsec:regret}
Next we present upper bound results regarding confidence and regret, two key measures besides total sampling cost. 

First, for LP2S based on LP-PAC, we have the following theorem.
\begin{theorem} \label{thm:LP_PAC}
	Suppose LP-PAC is feasible. Then we miss an $(1-\mu_0 + C_1 \sqrt{\log L/R})$-optimal arm 
	with probability at most
	$$C_2 e^{-(1-\delta_0)L}$$
	for absolute constants $C_1, C_2>0$.
	Moreover,
	$$\BSR \le 1-\mu_0 + C_1 \sqrt{\frac{\log L}{R}} + C_2 e^{-(1-\delta_0)L} .$$
\end{theorem}
We defer the proof of Theorem \ref{thm:LP_PAC}, as well as subsequent Theorems \ref{thm:LP_SRM} and \ref{thm:LP_FC} to Appendix \ref{sec:proof_LP_regret}. The three terms $1-\mu_0$, $C_1 \sqrt{\log L/R}$ and $C_2 e^{-(1-\delta_0)L}$ in the upper bound correspond to regret bound of any $(1-\mu_0)$-optimal arm, regret in stage 2, and the probability of no $(1-\mu_0)$-optimal arm arms chosen in stage 1, respectively. Note that the second term increases in $L$ and the last term decreases in $L$, so there is a trade-off in the choice of $L$.

\medskip

For LP2S based on LP-SRM, we provide the following results on the upper bound 
of $\BSR$.
\begin{theorem} \label{thm:LP_SRM}
	Suppose LP-SRM is feasible. We 
	have
	$$\BSR \le e^{-L} + 1-\delta_0.$$
\end{theorem}
The two terms $e^{-L}$ and $1-\delta_0$ correspond respectively to the probability of no arms surviving stage 1, and the regret bound in stage 2.

\medskip

For LP2S based on LP-FC, the following assumption is needed to derive upper bounds of $1-\BPB$ and $\BSR$.
\begin{assumption} \label{assump:FC}
	The c.d.f. $F_\pi(\cdot)$ of the prior $\pi$ satisfies
	\begin{enumerate}
		\item There exists an absolute constant $\alpha>0$ such that $F_\pi(1-d)\ge 1-d^{\alpha}$ for all $d$ close enough to 0. 
		\item $|F_\pi(u_1)-F_\pi(u_2)|\le \beta |u_1-u_2|$ for an absolute constant $\beta>0$ and any $u_1, u_2 \in [0,1]$.
	\end{enumerate}
\end{assumption} 
It can be shown that Assumption \ref{assump:FC} is satisfied for $\pi=\mathrm{Beta}(a,b)$ with $a\ge 1, b\ge 1$.

The following theorem provides upper bounds of $1-\BPB$ and $\BSR$ for the two-stage algorithm based on LP-FC.
\begin{theorem} \label{thm:LP_FC}
	Suppose LP-FC is feasible, and Assumption \ref{assump:FC} holds. Let $\alpha_0=\min(\alpha, 1)$, then we have
	$$1-\BPB \lesssim 1-(1-\delta_0)L + e^{-L} + C_1 K^{-(\alpha_0 c-2)} + C_2 L \exp\left(-\frac{RK^{-2c}}{4}\right)$$
	for and absolute constants $c>2/\alpha_0$, and $C_1, C_2>0$.
	Moreover,
	\begin{align*}
		\BSR \le \min & \left\{1-(1-\delta_0)L + e^{-L} + C_1 K^{-(\alpha_0 c-2)} + C_2 L \exp\left(-\frac{RK^{-2c}}{4}\right), \right. \\
		& \left. 1-(1-\delta_0)L + e^{-L} + C_3\sqrt{\frac{\log L}{R}} \right\}
	\end{align*}
	for absolute constant $C_3>0$.
\end{theorem}
The four terms $1-(1-\delta_0)L$, $e^{-L}$, $C_1 K^{-(\alpha_0 c-2)}$ and $C_2 L \exp\left(RK^{-2c}/4\right)$ in the upper bound of $1-\BPB$ correspond to the probability of not including the best arm in stage 1, the probability of no arms surviving stage 1, the probability that the top two arms are too close, and the probability of not selecting the optimal arm in stage 2, respectively. The BSR has two simultaneous upper bounds, of which the first is the same as the upper bound of $1-\BPB$, and the second is similar to the BSR bound in Theorem \ref{thm:LP_PAC}.


\section{Simulation Studies}  \label{sec:simulation}
In this section, we study the numerical performance of LP2S. We carry out three experiments, corresponding to applying
LP-SRM, LP-PAC and LP-FC in the first stage of LP2S respectively. The competitors include 1) two stage exploration 
(abbreviated TSE) from 
\cite{komiyama2021optimal}; 2) batched Thompson sampling from 
\cite{kalkanli2021abatched}; and 3) BatchRacing from \cite{jun2016top}.

In the first experiment, we apply LP-PAC in the first stage of LP2S.  We assume 
$\pi=\mathrm{Beta}(a,b)$, where $(a, b)=(1,1)$, $(5,1)$ or $(1,3)$. We set 
$K=1000$, $2000$ or $5000$. For LP2S, we take $R=c_1\log K$, $L=c_2\log K$ with 
$c_1=30$ and $c_2=3$, $\mu_0=0.7$ for $(a,b)=(1,1)$ or $(1,3)$, $\mu_0=0.8$ for $(a,b)=(5,1)$, and $\delta_0$ be the smallest number such that LP is feasible. For TSE, we fix $q=0.5$. For batched Thompson sampling, we take 
$\alpha=2$, and choose the arm with maximum average reward. In BatchRacing we allow maximum of $R$ batches, stick to the $(K,l)$-batch setting, and take confidence parameter $\delta=0.05$. BatchRacing works for top-$k$ 
arm identification, and we simply set $k=1$. The experiment runs $N=1000$ 
realizations of the $K$-arm bandit. We make the comparison in two different 
ways. First, we set the total sampling costs of the three competing methods 
approximately equal to that of LP2S, and compare the average simple regrets of 
the four methods over the $N$ simulation runs. Second, we set the simple 
regrets of the three competing methods close to that of LP2S, and compare the 
average total sampling costs of all methods. The average simple regret and average
sampling cost are reported in Table \ref{tab:LP-SRM}. For all cases, with the 
same sampling cost, LP2S has the smallest simple regret. Meanwhile, with similar 
simple regret, LP2S has the smallest sampling cost.

\begin{table}[]
	\centering
	\caption{Performances of LP-PAC induced LP2S and other competing methods}%
	\scalebox{1.0}{
		\renewcommand{\arraystretch}{1.2} {\footnotesize \
			\begin{tabular}{|c|c|c||c|c|c|c|}
				\hline
				\multirow{2}{*}{$(a,b)$} & \multirow{2}{*}{$K$} & 
				\multirow{2}{*}{$T$} & \multicolumn{4}{c|}{SR} \\ \cline{4-7}
				&  &  & LP2S & TSE & Batched Thompson & 
				BatchRacing \\ \hline
				\multirow{3}{*}{(1,1)} & 1000 & 6.94E+03 & {\bf 3.32E-03} & 
				1.11E-01 
				& 8.01E-02 & 9.34E-02 \\
				& 2000 & 1.19E+04 & {\bf 3.43E-03} & 1.25E-01 & 1.02E-01 & 
				1.34E-01 
				\\ 
				& 5000 & 1.48E+04 & {\bf 4.01E-03} & 1.90E-01 & 1.83E-01 & 
				2.15E-01 
				\\ \hline
				\multirow{3}{*}{(5,1)} & 1000 & 3.77E+03 & {\bf 5.58E-03} & 
				1.17E-01 & 9.72E-02 & 1.01E-01 \\
				& 2000 & 1.03E+04 & {\bf 3.16E-03} & 1.08E-01 & 8.65E-02 & 
				8.94E-02 \\
				& 5000 & 1.31E+04 & {\bf 4.14E-03} & 1.19E-01 & 1.13E-01 & 
				1.11E-01 
				\\ \hline
				\multirow{3}{*}{(1,3)} & 1000 & 1.11E+04 & {\bf 9.66E-03} & 
				4.92E-02
				& 1.38E-01 & 1.12E-01 \\
				& 2000 & 1.80E+04 & {\bf 5.76E-03} & 6.92E-02	& 2.04E-01 & 
				1.60E-01 \\
				& 5000 & 2.91E+04 & {\bf 7.85E-03} & 1.10E-01 & 3.75E-01 & 
				2.51E-01 
				\\ \hline \hline 
				\multirow{2}{*}{$(a,b)$} & \multirow{2}{*}{$K$} & 
				\multirow{2}{*}{SR} & \multicolumn{4}{c|}{$T$} \\ \cline{4-7}
				&  &  & LP2S & TSE & Batched Thompson & 
				BatchRacing \\ \hline
				\multirow{3}{*}{(1,1)} & 1000 & 3.32E-03 & {\bf 6.94E+03} & 
				1.80E+05 
				& 2.00E+04 & 1.97E+05 \\
				& 2000 & 3.43E-03 & {\bf 1.19E+04} & 4.00E+05 & 4.00E+04 & 
				4.30E+05 
				\\ 
				& 5000 & 4.01E-03 & {\bf 1.48E+04} & 8.00E+05 & 1.20E+05 & 
				8.43E+05 
				\\ \hline
				\multirow{3}{*}{(5,1)} & 1000 & 5.58E-03 & {\bf 3.77E+03} & 
				2.80E+05 
				& 2.20E+04 & 1.80E+05 \\
				& 2000 & 3.16E-03 & {\bf 1.03E+04} & 7.51E+05 & 5.00E+04 & 
				5.48E+05 \\
				& 5000 & 4.14E-03 & {\bf 1.31E+04} & 1.60E+06 & 1.20E+05 & 
				8.50E+05 
				\\ \hline
				\multirow{3}{*}{(1,3)} & 1000 & 9.66E-03 & {\bf 1.11E+04} & 
				6.00E+04 
				& 2.50E+05 & 1.40E+05 \\
				& 2000 & 5.76E-03 & {\bf 1.80E+04} & 1.20E+05 & 7.00E+05 & 
				2.70E+05 \\
				& 5000 & 7.85E-03 & {\bf 2.91E+04} & 3.00E+05 & 1.30E+06      & 
				9.89E+05  
				\\ \hline
	\end{tabular}}}
	\label{tab:LP-PAC}
\end{table}

In the second experiment, we use LP-SRM in the first stage of LP2S. All the parameters are the same as the first experiment, except that we do not need to specify $\mu_0$. The results on the average simple regret and sampling cost are shown in Table \ref{tab:LP-PAC}. We draw the same conclusion that the LP2S method has the lowest simple regret given the same sampling cost, and it has the lowest sampling cost given approximately the same simple regret.

\begin{table}[]
	\centering
	\caption{Performances of LP-SRM induced LP2S and other competing methods}%
	\scalebox{1.0}{
		\renewcommand{\arraystretch}{1.2} {\footnotesize \
			\begin{tabular}{|c|c|c||c|c|c|c|}
				\hline
				\multirow{2}{*}{$(a,b)$} & \multirow{2}{*}{$K$} & 
				\multirow{2}{*}{$T$} & \multicolumn{4}{c|}{SR} \\ \cline{4-7}
				&  &  & LP2S & TSE & Batched Thompson & 
				BatchRacing \\ \hline
				\multirow{3}{*}{(1,1)} & 1000 & 8.87E+03 & {\bf 3.74E-03} & 9.70E-02 & 5.09E-02 & 9.53E-02 \\
				& 2000 & 1.13E+04 &  {\bf 3.80E-03} & 1.31E-01 & 1.11E-01 & 1.40E-01 \\ 
				& 5000 & 1.41E+04 & {\bf 3.57E-03} & 1.89E-01 & 1.83E-01 & 2.15E-01 \\ \hline
				\multirow{3}{*}{(5,1)} & 1000 & 1.03E+04 & {\bf 4.54E-03} & 8.51E-02 &  3.75E-02 & 6.12E-02 \\
				& 2000 & 9.96E+03 & {\bf 3.96E-03} & 1.09E-01 & 8.68E-02 & 9.38E-02 \\
				& 5000 & 1.73E+04 & {\bf 2.16E-03} & 1.18E-01 & 1.03E-01 & 1.07E-01 \\ \hline
				\multirow{3}{*}{(1,3)} & 1000 & 5.19E+03 & {\bf 1.32E-01} & 1.49E-01 & 3.06E-01 & 2.47E-01 \\
				& 2000 & 6.60E+03 & {\bf 1.40E-01} & 2.46E-01 & 4.58E-01 & 3.46E-01  \\
				& 5000 & 1.26E+04 & {\bf 1.26E-01} & 2.65E-01 & 5.07E-01 & 4.06E-01 \\ \hline \hline 
				\multirow{2}{*}{$(a,b)$} & \multirow{2}{*}{$K$} & 
				\multirow{2}{*}{SR} & \multicolumn{4}{c|}{$T$} \\ \cline{4-7}
				&  &  & LP2S & TSE & Batched Thompson & 
				BatchRacing \\ \hline
				\multirow{3}{*}{(1,1)} & 1000 & 3.74E-03 & {\bf 8.87E+03} & 
				1.80E+05 & 2.00E+04 & 1.97E+05 \\ 
				& 2000 & 3.80E-03 & {\bf 1.13E+04} & 4.00E+05 & 4.00E+04 & 4.30E+05 \\ 
				& 5000 & 3.57E-03 & {\bf 1.41E+04} & 1.00E+06 & 1.20E+05 & 1.19E+06  \\ \hline
				\multirow{3}{*}{(5,1)} & 1000 & 4.54E-03 & {\bf 1.03E+04} & 3.00E+05 & 2.30E+04 & 2.00E+05 \\
				& 2000 & 3.96E-03 & {\bf 9.96E+03} & 7.51E+05 & 5.00E+04 & 5.48E+05  \\ 
				& 5000 & 2.16E-03 & {\bf 1.73E+04} & 2.00E+06 & 1.20E+05 & 1.27E+06 \\ \hline 
				\multirow{3}{*}{(1,3)} & 1000 & 1.32E-01 & {\bf 5.19E+03} & 5.59E+03 & 1.10E+04 & 9.11E+03 \\
				& 2000 & 1.40E-01 & {\bf 6.60E+03} & 1.10E+04 & 2.40E+04 & 2.03E+04 \\ 
				& 5000 & 1.26E-01 & {\bf 1.26E+04} & 3.74E+04 & 8.00E+04 & 6.01E+04 \\ \hline 	
	\end{tabular}}}
	\label{tab:LP-SRM}
\end{table}

In the last experiment, we test the performance of LP-FC induced LP2S. We 
maintain the same prior distribution as the first two experiments, but we only 
consider $K=200$. This is because LP-FC requires that $R\gtrsim K$, but when 
$R$ becomes large, the LP becomes too computationally expensive. For LP2S, we 
set $R=300$, $L=5$ and $\delta_0=0.93$. For the other three methods, we keep the 
same parameter settings. In addition to the two ways of comparison, we add 
another way of comparison, in which we set the PB of the competing methods the 
same as that of LP2S and compare their average sampling costs. Table 
\ref{tab:LP-FC} summarizes the average simple regret, sampling cost and 1-PB 
of all methods. Note the three blocks correspond to comparisons with fixed 
sampling cost, fixed simple regret and fixed PB, respectively. The performance of LP2S is only 
mediocre, inferior to TSE or batched Thompson 
sampling in most cases. This shows that LP-FC may not have advantages when $K$ is small.

\begin{table}[]
	\centering
	\caption{Performances of LP-FC induced LP2S and other competing methods. Batched Thompson is abbreviated BatchThomp to save space.}
	\scalebox{0.85}{
		\renewcommand{\arraystretch}{1.2} {\footnotesize \
	\begin{tabular}{|c|c|cccc||cccc|}
		\hline
		\multirow{2}{*}{$(a,b)$} &
		\multirow{2}{*}{$T$} &
		\multicolumn{4}{c||}{SR} &
		\multicolumn{4}{c|}{1-PB} \\ \cline{3-10} 
		&
		&
		\multicolumn{1}{c|}{LP2S} &
		\multicolumn{1}{c|}{TSE} &
		\multicolumn{1}{c|}{BatchThomp} &
		BatchRacing &
		\multicolumn{1}{c|}{LP2S} &
		\multicolumn{1}{c|}{TSE} &
		\multicolumn{1}{c|}{BatchThomp} &
		BatchRacing \\ \hline
		(1,1) &
		6.40E+03 &
		\multicolumn{1}{c|}{6.68E-03} &
		\multicolumn{1}{c|}{2.79E-02} &
		\multicolumn{1}{c|}{\textbf{2.48E-03}} &
		2.83E-02 &
		\multicolumn{1}{c|}{0.57} &
		\multicolumn{1}{c|}{0.93} &
		\multicolumn{1}{c|}{\textbf{0.22}} &
		0.83 \\
		(5,1) &
		\multicolumn{1}{l|}{8.68E+03} &
		\multicolumn{1}{l|}{9.34E-04} &
		\multicolumn{1}{c|}{3.52E-02} &
		\multicolumn{1}{c|}{\textbf{8.37E-04}} &
		1.91E-02 &
		\multicolumn{1}{c|}{\textbf{0.44}} &
		\multicolumn{1}{c|}{1.00} &
		\multicolumn{1}{c|}{0.46} &
		0.95 \\
		(1,3) &
		6.53E+03 &
		\multicolumn{1}{c|}{2.39E-02} &
		\multicolumn{1}{c|}{\textbf{1.80E-02}} &
		\multicolumn{1}{c|}{2.66E-02} &
		3.84E-02 &
		\multicolumn{1}{c|}{0.40} &
		\multicolumn{1}{c|}{0.40} &
		\multicolumn{1}{c|}{\textbf{0.17}} &
		0.50 \\ \hline
		\multirow{2}{*}{$(a,b)$} &
		\multirow{2}{*}{SR} &
		\multicolumn{4}{c||}{$T$} &
		\multicolumn{4}{c|}{1-PB} \\ \cline{3-10} 
		&
		&
		\multicolumn{1}{c|}{LP2S} &
		\multicolumn{1}{c|}{TSE} &
		\multicolumn{1}{c|}{BatchThomp} &
		BatchRacing &
		\multicolumn{1}{c|}{LP2S} &
		\multicolumn{1}{c|}{TSE} &
		\multicolumn{1}{c|}{BatchThomp} &
		BatchRacing \\ \hline
		(1,1) &
		6.68E-03 &
		\multicolumn{1}{c|}{6.40E+03} &
		\multicolumn{1}{c|}{1.71E+04} &
		\multicolumn{1}{c|}{\textbf{3.00E+03}} &
		1.92E+04 &
		\multicolumn{1}{c|}{0.57} &
		\multicolumn{1}{c|}{0.83} &
		\multicolumn{1}{c|}{\textbf{0.34}} &
		0.61 \\
		(5,1) &
		9.34E-04 &
		\multicolumn{1}{c|}{8.68E+03} &
		\multicolumn{1}{c|}{1.00E+05} &
		\multicolumn{1}{c|}{\textbf{7.00E+03}} &
		5.99E+04 &
		\multicolumn{1}{c|}{\textbf{0.44}} &
		\multicolumn{1}{c|}{0.88} &
		\multicolumn{1}{c|}{0.54} &
		0.71 \\
		(1,3) &
		2.39E-02 &
		\multicolumn{1}{c|}{6.53E+03} &
		\multicolumn{1}{c|}{5.10E+03} &
		\multicolumn{1}{c|}{\textbf{2.20E+03}} &
		9.00E+03 &
		\multicolumn{1}{c|}{0.40} &
		\multicolumn{1}{c|}{0.45} &
		\multicolumn{1}{c|}{\textbf{0.37}} &
		0.42 \\ \hline
		\multirow{2}{*}{$(a,b)$} &
		\multirow{2}{*}{1-PB} &
		\multicolumn{4}{c||}{$T$} &
		\multicolumn{4}{c|}{SR} \\ \cline{3-10} 
		&
		&
		\multicolumn{1}{c|}{LP2S} &
		\multicolumn{1}{c|}{TSE} &
		\multicolumn{1}{c|}{BatchThomp} &
		BatchRacing &
		\multicolumn{1}{c|}{LP2S} &
		\multicolumn{1}{c|}{TSE} &
		\multicolumn{1}{c|}{BatchThomp} &
		BatchRacing \\ \hline
		(1,1) &
		0.57 &
		\multicolumn{1}{c|}{6.40E+03} &
		\multicolumn{1}{c|}{3.60E+04} &
		\multicolumn{1}{c|}{\textbf{2.00E+03}} &
		2.16E+04 &
		\multicolumn{1}{c|}{6.68E-03} &
		\multicolumn{1}{c|}{\textbf{2.05E-03}} &
		\multicolumn{1}{c|}{2.13E-02} &
		5.33E-03 \\
		(5,1) &
		0.44 &
		\multicolumn{1}{c|}{\textbf{8.68E+03}} &
		\multicolumn{1}{c|}{3.00E+05} &
		\multicolumn{1}{c|}{9.00E+03} &
		5.99E+04 &
		\multicolumn{1}{c|}{9.34E-04} &
		\multicolumn{1}{c|}{\textbf{2.26E-04}} &
		\multicolumn{1}{c|}{6.82E-04} &
		2.25E-03 \\
		(1,3) &
		0.40 &
		\multicolumn{1}{c|}{6.53E+03} &
		\multicolumn{1}{c|}{6.52E+03} &
		\multicolumn{1}{c|}{\textbf{2.20E+03}} &
		1.11E+04 &
		\multicolumn{1}{c|}{2.39E-02} &
		\multicolumn{1}{c|}{\textbf{1.80E-02}} &
		\multicolumn{1}{c|}{6.31E-02} &
		2.12E-02 \\ \hline
\end{tabular}}}
\label{tab:LP-FC}
\end{table}


\section{Conclusion}  \label{sec:conclusion}

We propose a general LP framework, the first method of its kind in the literature, to solve BAI problem in batched Bayesian bandits. Such a framework is computationally tractable, and can broadly applied to the typical PAC, simple regret minimization, FB, and FC settings in BAI. The LP-induced two stage algorithm can find at least a good arm or even the best arm with small sampling cost, thus is especially useful for bandits with large $K$ and limited $R$. We show that proposed method have nice theoretical properties and good numerical performances. 

There are several directions for future research. First, this paper lacks analysis on the gap between optimal solutions of OPT-dep and OPT-ind. The key problem is how to quantify the information loss in ignoring the states of arms other than the focal arm. Second, the LP formulation can be possibly extended to the case where the rewards follow distributions other than Bernoulli. If rewards follow a discrete distribution, the state space is still discrete, then the generalization of LP is straightforward by using a multinomial tree to describe the state transitions of the focal arm. If rewards follow a continuous distribution, the state space becomes continuous, and how to formulate a tractable optimization problem is somewhat obscure. Finally, the method can be possibly generalized to contextual bandit problems, where the framework should allow transition probabilities $P(r,s)$ and actions $a(r,s)$ depend on the observed contextual information.

\newpage
\appendix

\section{Proof of Theorem \ref{thm:LP_thresh}}  
\label{sec:proof_LP_thresh}

To prove Theorem \ref{thm:LP_thresh}, we first present an sub-problem of the 
original LP. For $0\le r_0\le R-1$, $0\le s_0\le r_0$, suppose one starts from 
state $(r_0,s_0)$ with $P(r_0,s_0)=1$, consider the LP problem
\begin{align} \tag{LP($r_0,s_0,L_0$)}
	\begin{split} 
		\min & \sum_{r=r_0+1}^{R}\sum_{s=s_0}^{r} P(r,s) \\
		\mathrm{s.t.\ } & \eqref{eq:P0P1_sum} \mathrm{\ and\ } \eqref{eq:P_recursion} \mathrm{\ hold\ 
		for\ } r_0\le r \le R, s_0\le s\le r, \\
		 & \eqref{eq:survive_constraint}-\eqref{eq:conditional_constraint} 
		\mathrm{\ hold\ with \ some\ } L=L_0, \\
		& P_1(r_0,s_0)=1, \quad P_1(r+1,0)=0, \quad P_0(r,r)=0, \quad r_0\le r\le R, \\
		& P_1(r,s)=P_0(r,s)=0, \quad r_0\le r\le R, s\notin [s_0, s_0+r-r_0]. 
	\end{split}
\end{align}
Intuitively, LP($r_0,s_0,L_0$) represents optimization within a sub binomial 
tree starting from node $(r_0, s_0)$ of the entire binomial tree of the 
original LP. The second last line of constraints means that we assign 0 
probability to nodes outside the sub-tree. 

{\bf Combining.} For $0\le s_0\le r_0-1$, it is desirable to combine the 
sub-trees of LP($r_0+1,s_0,L_1$) and LP($r_0+1,s_0+1,L_1$) to get the sub-tree 
of LP($r_0,s_0,L_0$). Suppose for $0\le s_0\le r_0-1$, LP($r_0+1,s_0,L_1$) has 
a feasible solution
$$\left\{P'(r,s), P'_1(r,s), P'_0(r,s): r_0+1\le r\le R, s_0 \le s \le r 
\right\}$$
with objective value $C'(r_0+1, s_0, L_1)$. Suppose also LP($r_0+1,s_0+1,L_1$) 
has a feasible solution
$$\left\{P''(r,s), P''_1(r,s), P''_0(r,s): r_0+1\le r\le R, s_0+1 \le s \le r 
\right\}$$
with objective value $C''(r_0+1, s_0+1, L_1)$. For $a(r_0,s_0)\in [0,1]$, we 
let $L_0=a(r_0, s_0) L_1$ and 
\begin{align} \label{eq:feasible_r0s0}
	\begin{split}
		& P_1(r_0, s_0) = 1, \quad P_0(r_0, s_0)= 0 \\
		& P(r_0+1,s_0) = P_0(r_0+1,s_0)= a(r_0,s_0)\left\{\ 1-q(r_0,s_0) 
		\right\} \\
		& P(r_0+1,s_0+1) = P_1(r_0+1,s_0+1) = a(r_0,s_0) q(r_0,s_0) \\
		& P_0(r,s) = P(r_0+1, s_0) P_0'(r, s) + P(r_0+1, s_0+1) P_0''(r, s), 
		\quad r_0+2\le r\le R, s_0\le s\le r \\
		& P_1(r,s) = P(r_0+1, s_0) P_1'(r, s) + P(r_0+1, s_0+1) P_1''(r, s), 
		\quad r_0+2\le r\le R, s_0\le s\le r \\
		& P(r,s) = P_1(r,s)+ P_0(r,s), \quad r_0+2\le r\le R, s_0\le s\le r.
	\end{split}
\end{align}
The following lemma holds:
\begin{lemma} \label{lem:combine} 
	The probabilities \eqref{eq:feasible_r0s0} constitute a feasible solution 
	for LP($r_0,s_0,L_0$) with objective value $$C(r_0, s_0, L_0)=a(r_0, s_0)+ 
	P(r_0+1, s_0) C'(r_0+1, s_0, L_1) + P(r_0+1, s_0+1) C''(r_0+1, s_0+1, 
	L_1).$$
\end{lemma}
\begin{proof}
	First of all, the constraints of LP($r_0,s_0,L_0$) can be checked by basic 
	algebra. The objective value is
	\begin{align*}
		\sum_{r=r_0+1}^{R}\sum_{s=s_0}^r P(r,s) & = P(r_0+1,s_0) + 
		P(r_0+1,s_0+1) + \sum_{r=r_0+1}^{R}\sum_{s=s_0}^r P(r,s) \\
		& = a(r_0,s_0) + P(r_0+1,s_0) \sum_{r=r_0+2}^{R}\sum_{s=s_0}^{r} 
		P'(r,s) + P(r_0+1,s_0+1) \sum_{r=r_0+2}^{R}\sum_{s=s_0+1}^{r} P''(r,s) 
		\\
		& = a(r_0,s_0) + P(r_0+1,s_0) C'(r_0+1, s_0, L_1) + P(r_0+1,s_0+1) 
		C''(r_0+1, s_0+1, L_1).
	\end{align*}
\end{proof}

We denote the optimal solution of LP($r_0,s_0,L_0$) as $\bar{P}(r,s), 
\bar{P}_0(r,s), \bar{P}_1(r,s)$, with optimal value $\bar{C}(r_0,s_0,L_0)$. 
Then we have the following lemma:
\begin{lemma}\label{lem:sub_compare}
	For $0\le s_0\le r_0-1$, there exists a feasible solution to 
	LP($r_0,s_0+1,L_0$) with objective value $\wt{C}(r_0,s_0+1,L_0)$ such that 
	$\wt{C}(r_0,s_0+1,L_0) \le \bar{C}(r_0,s_0,L_0)$ and $\sum_{s=0}^R 
	w(s)\wt{P}(R,s)\ge \sum_{s=0}^R w(s)\bar{P}(R,s)$.
\end{lemma}
\begin{proof}
	We prove by induction on $r_0$. 
	
	First assume $r_0=R-1$. Suppose LP($R-1,s_0,L_0)$ induces an optimal action 
	$\bar{a}(R-1, s_0)$ and optimal probabilities 
	\begin{align*}
		\bar{P}(R,s_0) & =\bar{P_0}(R,s_0)=\bar{a}(R-1, s_0)\left\{1-q(R-1, 
		s_0)\right\}, \\
		\bar{P}(R,s_0+1) & =\bar{P_1}(R,s_0+1)=\bar{a}(R-1, s_0)q(R-1, s_0).
	\end{align*}
	In order to satisfy \eqref{eq:survive_constraint}, we need $\bar{P}(R,s_0)+ 
	\bar{P}(R,s_0+1)=\bar{a}(R-1,s_0)=L_0/K$. The optimal objective is also 
	$\bar{C}(R-1,s_0,L_0)= \bar{a}(R-1, s_0)=L_0/K$. Now for 
	LP($R-1,s_0+1,L_0)$, we apply action $\bar{a}(R-1, s_0)$ in state $(R-1, 
	s_0+1)$. We can get the feasible probabilities
	\begin{align*}
		\wt{P}(R,s_0+1) & =\wt{P_0}(R,s_0+1)=\bar{a}(R-1, s_0)\left\{1-q(R-1, 
		s_0+1)\right\}, \\
		\wt{P}(R,s_0+2) & =\wt{P_1}(R,s_0+2)=\bar{a}(R-1, s_0)q(R-1, s_0+1).
	\end{align*}
	We have $\wt{P}(R,s_0+1)+ \wt{P}(R,s_0+2)=\bar{a}(R-1, s_0)=L_0/K$, so 
	\eqref{eq:survive_constraint} is satisfied. Also since $w(s_0)\le w( s_0+1)\le w(s_0+2)$, we have
	\begin{align*}
		& w(s_0+1)\wt{P}(R,s_0+1)+w(s_0+2)\wt{P}(R,s_0+2) \ge 
		w(s_0+1)\left\{\wt{P}(R,s_0+1)+\wt{P}(R,s_0+2)\right\} \\
		& = w(s_0+1) \bar{a}(R-1, s_0) = 
		w(s_0+1)\left\{\bar{P}(R,s_0)+\bar{P}(R,s_0+1)\right\} \\
		& \ge w(s_0)\bar{P}(R,s_0)+w(s_0+1)\bar{P}(R,s_0+1) \ge 
		(1-\delta)L_0/K.
	\end{align*}
	So far we have obtained a feasible solution with the desired property.
	
	Now assume the conclusion holds for $r_0+1$ and any $0\le s_0\le r_0$. For 
	LP($r_0,s_0,L_0)$, suppose the optimal solution induces a instant action 
	$\bar{a}(r_0,s_0)$ and one-step probabilities $\bar{P}(r_0+1, s_0), 
	\bar{P}(r_0+1, s_0+1)$. We can apply Lemma \ref{lem:combine} to conclude 
	that
	$$\bar{C}(r_0, s_0, L_0)=\bar{a}(r_0, s_0)+ \bar{P}(r_0+1, s_0) 
	\bar{C}(r_0+1, s_0, L_1) + \bar{P}(r_0+1, s_0+1) \bar{C}(r_0+1, s_0+1, 
	L_1),$$
	where $L_1=L_0/\bar{a}(r_0,s_0)$. Now for LP($r_0,s_0+1,L_0)$, we apply the 
	instant action $\bar{a}(r_0,s_0)$ in state $(r_0, s_0+1)$ to get one-step 
	probabilities 
	\begin{align*}
		\wt{P}(r_0+1,s_0+1) & = \wt{P}_0(r_0+1,s_0+1)= \bar{a}(r_0,s_0)\left\{\ 
		1-q(r_0,s_0+1) \right\} \\
		\wt{P}(r_0+1,s_0+2) & = \wt{P}_1(r_0+1,s_0+2) = \bar{a}(r_0,s_0) 
		q(r_0,s_0+1).
	\end{align*}
	We then combine the optimal solutions of LP($r_0+1,s_0+1,L_1)$ and 
	LP($r_0+1,s_0+2,L_1)$ as in \eqref{eq:feasible_r0s0} to get the 
	probabilities for $r_0+2\le r\le R$. By Lemma \ref{lem:combine}, the 
	objective value for this feasible solution is
	\begin{align*}
		\wt{C}(r_0, s_0+1, L_0) & =\bar{a}(r_0, s_0)+ \wt{P}(r_0+1, s_0+1) 
		\bar{C}(r_0+1, s_0+1, L_1) + \wt{P}(r_0+1, s_0+2) \bar{C}(r_0+1, s_0+2, 
		L_1) \\
		& \le \bar{a}(r_0, s_0)+ \wt{P}(r_0+1, s_0+1) \bar{C}(r_0+1, s_0+1, 
		L_1) + \wt{P}(r_0+1, s_0+2) \bar{C}(r_0+1, s_0+1, L_1) \\
		& =  \bar{a}(r_0, s_0) + \bar{a}(r_0, s_0) \bar{C}(r_0+1, s_0+1, L_1) \\
		& \le \bar{a}(r_0, s_0)+ \bar{P}(r_0+1, s_0) \bar{C}(r_0+1, s_0, L_1) + 
		\bar{P}(r_0+1, s_0+1) \bar{C}(r_0+1, s_0+1, L_1) \\
		& = \bar{C}(r_0,s_0,L_0).
	\end{align*}
	Finally, denote the left-hand-side of \eqref{eq:conditional_constraint} for optimal 
	solutions of LP($r_0+1,s_0,L_1)$, LP($r_0+1,s_0+1,L_1)$ and 
	LP($r_0+1,s_0+2,L_1)$ as $\bar{D}(r_0+1,s_0,L_1)$, 
	$\bar{D}(r_0+1,s_0+1,L_1)$, $\bar{D}(r_0+1,s_0+2,L_1)$ respectively. We 
	then get
	\begin{align*}
		\sum_{s=0}^R w(s) \wt{P}(R,s) & = \wt{P}(r_0+1, s_0+1) 
		\bar{D}(r_0+1,s_0+1,L_1)+ \wt{P}(r_0+1, s_0+2) \bar{D}(r_0+1,s_0+2,L_1) 
		\\
		& \ge \wt{P}(r_0+1, s_0+1) \bar{D}(r_0+1,s_0+1,L_1)+ \wt{P}(r_0+1, 
		s_0+2) \bar{D}(r_0+1,s_0+1,L_1) \\
		& = \bar{a}(r_0,s_0) \bar{D}(r_0+1,s_0+1,L_1) \\
		& \ge \bar{P}(r_0+1, s_0) \bar{D}(r_0+1,s_0,L_1)+ \bar{P}(r_0+1, s_0+1) 
		\bar{D}(r_0+1,s_0+1,L_1) \\
		& = \sum_{s=0}^R w(s) \bar{P}(R,s).
	\end{align*}
\end{proof}

We now turn to the proof of Theorem \ref{thm:LP_thresh}.
\begin{proof}
	Let $\left\{\bar{P}(r,s), \bar{P}_1(r,s), \bar{P}_0(r,s): 0\le r\le R, 0\le 
	s\le r \right\}$ be an optimal solution of the original LP. In view of 
	Lemma \ref{lem:combine}, the optimal solution can be formed by combining 
	the optimal solutions of LP($r,s,L(r)$) and LP($r,s+1,L(r)$) as in 
	\eqref{eq:feasible_r0s0} backwards for $r=R-1, R-2, \dots, 1$ with 
	appropriately chosen $L(r)$.
	
	Let $\left\{\bar{a}(r,s): 0\le r\le R-1, 0\le s\le r \right\}$ be the 
	induced optimal actions. Assume there exists $(r_0, s_0)$ such that 
	$\bar{a}(r_0,s_0)\in (0,1), \bar{a}(r_0,s_0+1)\in (0,1)$, then we can find 
	a small $\epsilon>0$ such that 
	\begin{align*}
		\bar{a}(r_0,s_0) \bar{P}(r_0,s_0)-\epsilon & \ge 0 \\
		\bar{a}(r_0,s_0+1) \bar{P}(r_0,s_0+1)+\epsilon  & \le 1.
	\end{align*}
	We can find a new feasible solution by decreasing $\bar{a}(r_0,s_0) 
	\bar{P}(r_0,s_0)$ by $\epsilon$ and increasing $\bar{a}(r_0,s_0+1) 
	\bar{P}(r_0,s_0+1)$ by $\epsilon$. As a result, $P_0(r_0+1, s_0)$, 
	$P_1(r_0+1, s_0+1)$, $P_0(r_0+1, s_0+1)$ and $P_1(r_0+1, s_0+2)$ will 
	change by $\Delta_{1}=-\epsilon \left\{1-q(r_0, s_0)\right\}$, 
	$\Delta_{2}=-\epsilon q(r_0, s_0)$, $\Delta_{3}=\epsilon \left\{1-q(r_0, 
	s_0+1)\right\}$ and $\Delta_{4}=\epsilon q(r_0, s_0+1)$ respectively. 
	Apparently $\Delta_1+\Delta_2+\Delta_3+\Delta_4=0$. Assume 
	LP($r_0+1,s_0,L(r_0+1)$) has optimal solution
	$$\left\{\bar{P}'(r,s), \bar{P}'_1(r,s), \bar{P}'_0(r,s): r_0+1\le r\le R, 
	s_0 \le s \le r \right\}$$
	with optimal objective value $\bar{C}(r_0+1, s_0, L(r_0+1))$, 
	LP($r_0+1,s_0+1,L(r_0+1)$) has optimal solution
	$$\left\{\bar{P}''(r,s), \bar{P}''_1(r,s), \bar{P}''_0(r,s): r_0+1\le r\le 
	R, s_0+1 \le s \le r \right\}$$
	with optimal objective value $\bar{C}(r_0+1, s_0+1, L(r_0+1))$, and 
	LP($r_0+1,s_0+2,L(r_0+1)$) has optimal solution
	$$\left\{\bar{P}'''(r,s), \bar{P}'''_1(r,s), \bar{P}'''_0(r,s): r_0+1\le 
	r\le R, s_0+2 \le s \le r \right\}$$
	with optimal objective value $\bar{C}(r_0+1, s_0+2, L(r_0+1))$. We define 
	the set of probabilities $\wt{P}(r,s)$, $\wt{P}_1(r,s)$, $\wt{P}_0(r,s)$ as 
	follows. For $0\le r\le r_0$, let
	\begin{align*}
		\wt{P}(r,s) =\bar{P}(r,s), \quad \wt{P}_1(r,s) =\bar{P}_1(r,s), \quad 
		\wt{P}_0(r,s) =\bar{P}_0(r,s), \quad 0\le s\le r.
	\end{align*}
	For $r=r_0+1$, let
	\begin{align*}
		& \wt{P}_0(r_0+1,s_0) = \bar{P}_0(r_0+1,s_0) + \Delta_1 \\
		& \wt{P}_1(r_0+1,s_0+1) = \bar{P}_1(r_0+1,s_0+1) + \Delta_2 \\
		& \wt{P}_0(r_0+1,s_0+1) = \bar{P}_0(r_0+1,s_0+1) + \Delta_3 \\
		& \wt{P}_1(r_0+1,s_0+2) = \bar{P}_1(r_0+1,s_0+2) + \Delta_4 \\
		& \wt{P}_1(r_0+1,s) =\bar{P}_1(r_0+1,s), \quad \wt{P}_0(r,s) 
		=\bar{P}_0(r_0+1,s), \quad s \notin \{s_0, s_0+1, s_0+2\} \\
		& \wt{P}(r_0+1,s) = \wt{P}_0(r_0+1,s)+ \wt{P}_1(r_0+1,s), \quad 0\le 
		s\le r_0+1.
	\end{align*}
	For $r\ge r_0+2$ and $0\le s\le r$, let
	\begin{align*}
		& \wt{P}_0(r,s) = \bar{P}_0(r,s) + \Delta_1 P_0'(r, s) + 
		(\Delta_2+\Delta_3) P_0''(r, s)+ \Delta_4 P_0'''(r,s) \\
		& \wt{P}_1(r,s) = \bar{P}_1(r,s) + \Delta_1 P_1'(r, s) + 
		(\Delta_2+\Delta_3) P_1''(r, s)+ \Delta_4 P_1'''(r,s) \\
		& \wt{P}(r,s) = \wt{P}_1(r,s)+ \wt{P}_0(r,s).
	\end{align*}
	It can be verified that for $\wt{P}(r,s), \wt{P}_1(r,s), \wt{P}_0(r,s)$, 
	constraints \eqref{eq:first_constraint}--\eqref{eq:boundary_constraints} are maintained. Also, we have
	\begin{align*}
		\sum_{s=0}^R \wt{P}(R,s) & = \sum_{s=0}^R \bar{P}(R,s) + \Delta_1 
		\sum_{s=0}^R P'(R,s) +(\Delta_2+\Delta_3) \sum_{s=0}^R 
		P''(R,s)+\Delta_4 \sum_{s=0}^R P'''(R,s) \\
		& = \frac{L}{K} + \Delta_1 \frac{L(r_0+1)}{K} + 
		(\Delta_2+\Delta_3)\frac{L(r_0+1)}{K}+ \Delta_4 \frac{L(r_0+1)}{K} \\
		& = \frac{L}{K},
	\end{align*}
	and 
	\begin{align*}
		& \sum_{s=0}^R w(s) \wt{P}(R,s) \\
		& = \sum_{s=0}^R w(s) \bar{P}(R,s) + \Delta_1 \sum_{s=0}^R w(s) 
		P'(R,s) +(\Delta_2+\Delta_3) \sum_{s=0}^R w(s)P''(R,s)+\Delta_4 
		\sum_{s=0}^R w(s)P'''(R,s) \\
		& \ge \frac{(1-\delta_0)L}{K} + (\Delta_1+\Delta_2)\sum_{s=0}^R 
		w(s)P''(R,s) + (\Delta_3+\Delta_4)\sum_{s=0}^R w(s)P''(R,s) \\
		& = \frac{(1-\delta_0)L}{K}.
	\end{align*}
	Furthermore, we have 
	\begin{align*}
		& \sum_{r=1}^R\sum_{s=0}^r \wt{P}(r,s) \\
		& = \sum_{r=1}^R\sum_{s=0}^r \bar{P}(r,s) + \Delta_1 \bar{C}(r_0+1, 
		s_0, L(r_0+1)) + (\Delta_2+\Delta_3) \bar{C}(r_0+1, s_0+1, L(r_0+1)) \\
		& +\Delta_4 \bar{C}(r_0+1, s_0+2, L(r_0+1)) \\
		& \le \sum_{r=1}^R\sum_{s=0}^r \bar{P}(r,s) + (\Delta_1+\Delta_2) 
		\bar{C}(r_0+1, s_0+1, L(r_0+1)) + (\Delta_3+\Delta_4) \bar{C}(r_0+1, 
		s_0+1, L(r_0+1)) \\
		& = \sum_{r=1}^R\sum_{s=0}^r \bar{P}(r,s).
	\end{align*}
	Therefore we have found a feasible solution with cost no more than the 
	optimal solution. Applying this argument repeatedly we can eventually reach 
	a solution where either $a(r_0, s_0)=0$ or $a(r_0, s_0+1)=1$. 
\end{proof}

\section{Proof of Theorem \ref{thm:LP_cost} and Corollary \ref{coro:LP_cost}}  
\noindent {\bf Proof of Theorem \ref{thm:LP_cost}.}
\label{sec:proof_LP_cost}
\begin{proof}
	Since $w(s)$ is non-decreasing in $s$, we have 
	\begin{align*}
		w(R) \sum_{s=0}^R P(R,s) \ge \sum_{s=0}^R w(s) P(R,s) \ge (1-\delta_0) 
		\sum_{s=0}^R P(R,s)
	\end{align*}
	In order to be feasible, it must be that $w(R)\ge 1-\delta_0$. 
	
	Now define a policy $\wh{A}=\left\{\wh{a}(r,s): 0\le r\le R, 0\le s\le 
	r\right\}$
	where 
	\begin{align*}
		\wh{a}(r,s) =\begin{cases}
			1, & \mathrm{\ if\ } s=r \\
			0, & \mathrm{\ if\ } 0\le s\le r-1
		\end{cases}
	\end{align*}
	for $1\le r\le R$, and 
	$$\wh{a}(0,0)=\frac{L}{K \Expect_{\pi} \left(\mu^R\right)}.$$
	Then policy $\wh{A}$ generates a set of probabilities $\left\{\wh{P}(r,s), 
	\wh{P}_1(r,s), \wh{P}_0(r,s): 0\le r\le R, 0\le s\le r \right\}$ satisfying 
	\eqref{eq:P_recursion}, 
	with 
	\begin{align*}
		\wh{P}(r,s) =\begin{cases}
			\wh{a}(0,0) \Expect_{\pi} \left(\mu^r\right), & \mathrm{\ if\ } s=r 
			\\
			0, & \mathrm{\ if\ } 0\le s\le r-1.
		\end{cases}
	\end{align*}
	Apparently, $\sum_{s=0}^R \wh{P}(R,s)=K/L$ and $\sum_{s=0}^R w(s) 
	\wh{P}(R,s)=w(R)K/L\ge (1-\delta_0)K/L$. Hence $\left\{\wh{P}(r,s), 
	\wh{P}_1(r,s), \wh{P}_0(r,s): 0\le r\le R, 0\le s\le r \right\}$ is 
	feasible. Its objective value is
	\begin{align*}
		\wh{f}=\sum_{r=1}^R \wh{P}(r,r) = \wh{a}(0,0) \sum_{r=1}^R 
		\Expect_{\pi} \left(\mu^r\right) = 
		\frac{L}{K\Expect_{\pi}(\mu^R)}\sum_{r=1}^R 
		\Expect_{\pi}(\mu^r).
	\end{align*}
\end{proof}

\noindent {\bf Proof of Corollary \ref{coro:LP_cost}.}
\begin{proof}
	For $1\le r\le R$ we have
	\begin{align*}
		\Expect_{\pi}(\mu^r)= \frac{B(a+r,b)}{B(a,b)} = \frac{\Gamma(a+r)\Gamma(a+b)}{\Gamma(a+b+r)\Gamma(a)}.
	\end{align*}
	Applying the asymptotic approximation $\Gamma(r+c)\sim \Gamma(r) r^c$ as $r\to \infty$, we get
	\begin{align*}
		\Expect_{\pi}(\mu^r)\sim \frac{\Gamma(a+b)}{\Gamma(a)}r^{-b}.
	\end{align*}
	The conclusion follows by noting the asymptotic equivalence of the sum and integral.
\end{proof}

\section{Proof of Theorems \ref{thm:LP_PAC}, \ref{thm:LP_SRM} and \ref{thm:LP_FC}}
\label{sec:proof_LP_regret}

\noindent {\bf Proof of Theorem \ref{thm:LP_PAC}.}
\begin{proof}
	Let $J=|\bbJ_{R}|$ be the number of arms surviving from stage 1. 
	The same as before, we define the events
	\begin{align*}
		E_J & =\left\{ J \mathrm{\ arms \ survive\ in\ stage\ 1} \right\}, \quad 0\le J\le K, 
		\\
	\end{align*}
	and let $\nu^*=\max_{\left\{j\in \bbJ_{R}\right\}} \mu_j$ if $J>0$ and 
	$\nu^*=0$ if $J=0$.
	
	According to constraint \eqref{eq:conditional_constraint} for LP-PAC, the probability that 
	an arm has $\mu<\mu_0$ conditional on the event that the arm survives is 
	upper bounded by $\delta_0$. By independence between arms, we have
	\begin{align}
		\Prob(\nu^*<\mu_0 \mid E_J) \le \delta_0^J, \quad 1\le J\le K. 
		\label{eq:regret_stage1}
	\end{align}
	The overall probability of not choosing any arm with $\mu\ge \mu_0$ in stage 1 is 
	\begin{align*}
		\delta^* & =\Prob(E_0) + \sum_{J=1}^K \Prob(E_J) \Prob\left(\nu^*<\mu_0 \mid E_J\right) \\
		& \le \left(1-L/K\right)^K + \Expect(\delta_0^J) \\
		& \sim e^{-L} + e^{-(1-\delta_0)L} \lesssim e^{-(1-\delta_0)L}.
	\end{align*}
	
	For stage 2, we apply Theorem 33.1 and its subsequent discussion in 
	\cite{lattimore2020bandit} to get 
	\begin{align}
		\Expect\left(\nu^*-\mu_{\wh{j}}\mid E_J \cap \left\{\nu^*\ge 
		\mu_0\right\}\right) \le O\left(\sqrt{\frac{\log J}{R}}\right), \quad 
		1\le J\le K \label{eq:regret_stage2}
	\end{align}
	Conditional on $E_J \cap \left\{\nu^*\ge \mu_0\right\}$, the regret can be upper bounded by
	\begin{align*}
		\Expect\left(\mu^*-\mu_{\wh{j}}\mid E_J \cap \left\{\nu^*\ge 
		\mu_0\right\}\right) & = \Expect\left(\mu^*-\nu^*\mid E_J \cap \left\{\nu^*\ge 
		\mu_0\right\}\right) + \Expect\left(\nu^*-\mu_{\wh{j}}\mid E_J \cap \left\{\nu^*\ge 
		\mu_0\right\}\right) \\
		& \le 1-\mu_0 + O\left(\sqrt{\frac{\log J}{R}}\right).
	\end{align*}
	Therefore, given that an arm with $\mu\ge \mu_0$ survives in stage 1, the regret is upper bounded by 
	\begin{align*}
		& \sum_{J=1}^K \Prob(E_J) \Prob(\nu^*\ge \mu_0\mid E_J) \Expect\left(\mu^*-\mu_{\wh{j}}\mid E_J \cap \left\{\nu^*\ge \mu_0\right\}\right) \\
		& \le 1-\mu_0 + \Expect\left\{O\left(\sqrt{\frac{\log J}{R}}\right)\right\} \\
		& \le 1-\mu_0 + O\left(\sqrt{\frac{\log \Expect(J)}{R}}\right) \\
		& = 1-\mu_0 + O\left(\sqrt{\frac{\log L}{R}}\right),
	\end{align*}
	where the last inequality follows Jensen's inequality. Thus the PAC conclusion holds.
	
	Note that in $E_0$ or $\{\nu^*<\mu_0\}$, regret can be upper bounded by 1. 
	Hence, BSR satisfies
	\begin{align*}
		\BSR & \le \delta^* +  1-\mu_0 + O\left(\sqrt{\frac{\log L}{R}}\right). 
	\end{align*} 
\end{proof}

\noindent {\bf Proof of Theorem \ref{thm:LP_SRM}.}
\begin{proof}
	We adopt the same notations $E_J$ and $\nu^*$ as in the proof of Theorem 
	\ref{thm:LP_PAC}. Then
	\begin{align*}
		\BSR \le \Prob(E_0) + \sum_{J=1}^K \Prob(E_J) \Expect\left(\mu^*-\mu_{\wh{j}}\mid E_J \right) \le e^{-L} + 1-\delta_0,
	\end{align*}
	where the second inequality is a result of the additional constraint 2 for LP-SRM.
\end{proof}

\medskip

\noindent {\bf Proof of Theorem \ref{thm:LP_FC}.}
\begin{proof}
	Let $J=|\bbJ_{R}|$ be the number of arms surviving from the first stage. 
	Then $J\sim \mathrm{Bin}(K, L/K)$. Define the events
	\begin{align*}
		E_J & =\left\{ J \mathrm{\ arms \ survive\ in\ stage\ 1 } \right\} \\
		E^1 & =\left\{\mathrm{no\ surviving\ arms\ in\ stage\ 1 \
			satisfy \ } \mu=\mu^*\right\} \\
		E^2 & =\left\{\mathrm{one \ surviving\ arm\ in\ stage\ 1 \ 
			satisfies \ } \mu=\mu^*\right\} \\
		E^3 & =\left\{\mathrm{the\ recommended\ } \wh{j} 
		\mathrm{\ in\ stage\ 2\ is\ not \ best\ among\ the\ surviving\ arms} 
		\right\}
	\end{align*}
	for $0\le J\le K$.
	Then,
	\begin{align} \label{eq:1-BPB}
		1-\BPB\le \Prob(E_0) + \sum_{J=1}^K \Prob(E_J) \Prob(E^1 
		\mid E_J) + 
		\sum_{J=1}^K \Prob(E_J) \Prob(E^2 \mid E_J) \Prob(E^3\mid E_J\cap 
		E^2).
	\end{align}
	We derive bounds for the three terms in \eqref{eq:1-BPB} separately. First of all, we have 
	\begin{align} \label{eq:E_0}
		\Prob(E_0)= (1-L/K)^K \sim e^{-L}.
	\end{align}
	Secondly, LP-FC suggests that the probability of being the best given any 
	arm has survived stage 1 is at least $1-\delta_0$. Hence $\Prob(E^1 \mid 
	E_J) \le 1-(1-\delta_0)J$. Therefore,
	\begin{align} \label{eq:E1}
		\sum_{J=1}^K \Prob(E_J) \Prob(E^1 \mid E_J) \le 1-(1-\delta_0)\Expect 
		(J)
		= 1-(1-\delta_0)L.
	\end{align}
	Conditional on $E_J$, among the $J$ surviving arms, let $\wh{\Delta}_{2}$ 
	be the gap between the 
	expected rewards of the best and second best arm. Also let 
	$\Delta_2=\mu^*-\mu_{(2)}$,
	the gap between the 
	expected rewards of the best and second best arm among the original $K$ 
	arms. Conditional on $E^2$, $\wh{\Delta}_2\ge \Delta_2$ must hold. By 
	\cite{lattimore2020bandit}, we have
	\begin{align*}
		\Prob(E^3 \mid E_J\cap E^2) \le J 
		\exp\left(-\frac{R\wh{\Delta}_2^2}{4}\right) \le J 
		\exp\left(-\frac{R \Delta_2^2}{4}\right).
	\end{align*}
	Hence,
	\begin{align} \label{eq:1-BPB-term3}
		\sum_{J=1}^K \Prob(E_J) \Prob(E^2 \mid E_J) \Prob(E^3\mid E_J\cap 
		E^2)\le \Expect \left\{J \exp\left(-\frac{R \Delta_2^2}{4}\right)\right\}.
	\end{align}
	
	We know that the joint density of $(U,V)=(\mu_{(K-1)}, \mu^*)$ is
	\begin{align*}
		F_{U,V}(u,v)=K(K-1)F^{K-2}_{\pi}(u) f_{\pi}(u) f_{\pi}(v), \quad 0\le 
		u\le v\le 1.
	\end{align*}
	Then for small $d>0$, we get
	\begin{align*}
		& \Prob(\Delta_2 \le d) =\int_{v-u\le d} F_{U,V}(u,v) du dv \\
		& = \int_0^{1-d} K(K-1)F^{K-2}_{\pi}(u) f_{\pi}(u) du \int_{u}^{u+d} 
		f_{\pi}(v) dv +\int_{1-d}^{1} K(K-1)F^{K-2}_{\pi}(u) f_{\pi}(u) du 
		\int_{u}^{1} f_{\pi}(v) dv \\
		& = \int_0^{1-d} K(K-1)F^{K-2}_{\pi}(u) 
		\left\{F_{\pi}(u+d)-F_{\pi}(u)\right\} dF_{\pi}(u) +\int_{1-d}^{1} 
		K(K-1)F^{K-2}_{\pi}(u) \left\{1- F_{\pi}(u)\right\} dF_{\pi}(u) \\
		& \le \int_0^{1-d} K(K-1)F^{K-2}_{\pi}(u)
		d \beta dF_{\pi}(u) +\int_{1-d}^{1} 
		K(K-1)F^{K-2}_{\pi}(u) \left\{1- F_{\pi}(u)\right\} dF_{\pi}(u) \\
		& = \beta K d F_{\pi}^{K-1}(1-d) + K\left\{1-F_{\pi}^{K-1}(1-d)\right\} 
		- (K-1)\left\{1-F_{\pi}^{K}(1-d)\right\} \\
		& \le 1-(1-\beta Kd)F_{\pi}^{K}(1-d) \\
		& \le 1-(1-\beta Kd)(1-d^\alpha)^K,
	\end{align*}
	where in the first and last inequalities we used Assumption 
	\ref{assump:FC}. Now take $d=K^{-c}$ with $c>2/\alpha_0$, then
	\begin{align*}
		(1-d^\alpha)^K = (1-K^{-\alpha c})^K \sim \exp\left(-K^{-(\alpha 
			c-1)}\right) \sim 1-K^{-(\alpha c-1)}.
	\end{align*}
	We in turn get 
	\begin{align*}
		\Prob(\Delta_2 \le d) & \lesssim 1-\left\{1-\beta 
		K^{-(c-1)}\right\}\left\{ 
		1-K^{-(\alpha c-1)}\right\} \lesssim K^{-(\alpha_0 c-1)}.
	\end{align*}
	Following \eqref{eq:1-BPB-term3}, we continue to get
	\begin{align}
		\Expect \left\{J \exp\left(-\frac{R \Delta_2^2}{4}\right)\right\} & = \Expect \left\{J \exp\left(-\frac{R \Delta_2^2}{4}\right)\indc{\Delta_2\le d}\right\} + \Expect \left\{J \exp\left(-\frac{R \Delta_2^2}{4}\right)\indc{\Delta_2> d} \right\} \nonumber \\
		& \le K \Prob(\Delta_2\le d) + \Expect(J) \exp\left(-\frac{R d^2}{4}\right) \nonumber \\
		& \le C_1 K^{-(\alpha_0 c-2)}+ 	C_2 L \exp\left(-\frac{R K^{-2c}}{4}\right). \label{eq:1-BPB-term3-1}
	\end{align}
	Combining \eqref{eq:1-BPB}--\eqref{eq:1-BPB-term3-1}, we get the desired result on 
	$1-\BPB$.
	
	\medskip
	
	In terms of $\BSR$, the upper bound of $1-\BPB$ naturally serves as its upper bound. We also have
	\begin{align*}
		\BSR \le \Prob(E_0) + \sum_{J=1}^K \Prob(E_J) \Prob(E^1\mid E_J) + 
		\sum_{J=1}^K \Prob(E_J) \Prob(E^2\mid E_J) 
		\Expect\left(\mu^*-\mu_{\wh{j}}\mid E_J\cap E^2\right).
	\end{align*}	
	For stage 2, we apply Theorem 33.1 and its subsequent discussion in 
	\cite{lattimore2020bandit} to get 
	\begin{align}
		\Expect\left(\mu^*-\mu_{\wh{j}}\mid E_J \cap E^2 \right) \le 
		O\left(\sqrt{\frac{\log J}{R}}\right), \quad 
		1\le J\le K. 
	\end{align}
	Then 
	\begin{align*}
		\BSR & \lesssim e^{-L} + 1-(1-\delta_0)L + 
		\Expect\left\{O\left(\sqrt{\frac{\log 
				J}{R}}\right)\right\} \\
		& \le e^{-L} + 1-(1-\delta_0)L + O\left(\sqrt{\frac{\log L}{R}}\right).
	\end{align*}
	where the last inequality is based on Jensen's inequality.
\end{proof}

\vskip 0.2in
\bibliographystyle{abbrvnat}
\bibliography{../mab}

\end{document}